\documentclass{uai2024_arxiv} 
                        
\usepackage[british]{babel}

\usepackage[T1]{fontenc}

\usepackage{natbib} 
\usepackage{mathtools} 
\usepackage{booktabs} 
\usepackage{tikz} 

\usepackage{graphicx}
\usepackage{balance} 
\usepackage{tcolorbox}
\usepackage{alltt}
\usepackage{array}
\usepackage{amsmath}
\usepackage[whole]{bxcjkjatype}
\usepackage{comment}
\usepackage{bm}
\usepackage{amssymb}
\usepackage{cases}

\usepackage{tikz}
\usepackage{pgfplots}
\usepackage{amsmath,amssymb,amsfonts,amsthm}

\newtheorem{example}{Example}
\newtheorem{theorem}{Theorem}
\newtheorem{definition}{Definition}

\newtheorem{proposition}{Proposition}
\newtheorem{corollary}{Corollary}
\newcommand{\cent}{\mathrel{\scalebox{1}[1.5]{$\shortmid$}\mkern-3.1mu\raisebox{0.1ex}{$=$}}}
\newcommand{\ent}{\mathrel{\scalebox{1}[1.5]{$\shortmid$}\mkern-3.1mu\raisebox{0.1ex}{$\equiv$}}}

\usepackage{stmaryrd}
\usepackage{bm}
\usepackage[whole]{bxcjkjatype}
\usepackage{cancel}

\newcommand{\ms}[1]{[\![#1]\!]}
\newcommand{\pms}[1]{[\![\![#1]\!]\!]}

\newcommand{\imgeq}[1]{\vcenter{\hbox{\includegraphics[scale=0.25]{#1}}}}

\newcommand{\imgeqL}[1]{\vcenter{\hbox{\includegraphics[scale=0.35]{#1}}}}

\title{Inference of Abstraction for a Unified Account of Reasoning and Learning}

\author[ ]{Hiroyuki Kido}
\affil[ ]{%
    School of Computer Science and Informatics\\
    Cardiff University\\
    Park Place, Cardiff, CF10 3AT, UK
}
  
\begin{document}
\maketitle

\begin{abstract}
Inspired by Bayesian approaches to brain function in neuroscience, we give a simple theory of probabilistic inference for a unified account of reasoning and learning. We simply model how data cause symbolic knowledge in terms of its satisfiability in formal logic. The underlying idea is that reasoning is a process of deriving symbolic knowledge from data via abstraction, i.e., selective ignorance. The logical consequence relation is discussed for its proof-based theoretical correctness. The MNIST dataset is discussed for its experiment-based empirical correctness.
\end{abstract}
\section{Introduction}\label{sec:intro}
Bayes' theorem plays an important role today in AI, neuroscience and cognitive science. It underlies most modern approaches to uncertain reasoning in AI systems \citep{russell:09}. Neuroscience often uses it as a metaphor for functions of the cerebral cortex, the outer portion of the brain in charge of higher-order cognitive functions such as perception, memory, emotion and thought \citep{lee:03,knill:04,george:05,colombo:12,funamizu:16}. It relates various brain theories such as Bayesian coding hypothesis \citep{knill:04}, free-energy principle \citep{friston:12} and predictive coding \citep{rao:99}. Their common idea is that the biological brain can be seen as a probabilistic generative model by which the past experience of the brain is constantly, but unconsciously, used to predict what is likely to happen outside the brain \citep{friston:12,sanborn:16,Hohwy:14}.
\par
The success of Bayesian approaches in AI and neuroscience leads to the expectation that there is a common Bayesian account of reasoning and learning, especially entailment (or deduction) and prediction, the main concern of formal logic and machine learning, respectively. The idea is worth investigating as it may give a clue to think upon how reasoning and learning operate in the human brain. Additionally, finding a principle underlying reasoning and learning is an open problem in AI across the different disciplines, e.g., neuro-symbolic AI \citep{Andreas:16}. Despite the scientific importance, few research in AI has focused on a Bayesian approach to a computational model of reasoning and learning. Indeed, most present research, e.g., \citep{nilsson:86,pearl:88,pearl:91,friedman:99,richardson:06,sato:95,matthias:13,sanfilippo:18,botha:19}, combines different methods for reasoning and learning for a practical purpose to deal with reasoning from uncertain source of information. For example, maximum likelihood estimation is the method most often used to learn the probability or weight of each symbolic knowledge. Logical semantics is then used to draw conclusions from the probabilistic or weighted symbolic knowledge. Various types of logical semantics exist such as the semantics of Bayesian networks, Markov logic networks and distribution semantics. However, the method used for learning cannot be used for reasoning, and vice versa. Moreover, in computational cognitive science, the theory-based Bayesian models of induction \citep{Tenenbaum:06}, the learned inference model \citep{Dasgupta:20} and the Bayesian program learning framework \citep{Lake:15,Lake:17} rest on the idea that observable data and their variants are generated from more abstract hypotheses such as background knowledge and principles about the world. Although the idea is prevalent in the machine learning community, the idea eventually struggles with intractable computation associated with an exponentially growing hypothesis space, especially when trying to incorporate symbolic knowledge. 
\par
In this paper, we report that some important aspects of reasoning and learning can be unified via inference of abstraction as selective ignorance. The simple idea underlying the inference of abstraction is that intrinsically abstract symbolic knowledge should be derived from intrinsically concrete data as a result of inference. The idea is simply formalised as a probabilistic model of the causality that data determine states of the world, and the states of the world determine the truth value of symbolic knowledge. The idea opposes \citep{Tenenbaum:06}, \citep{Dasgupta:20} and \citep{Lake:15,Lake:17}, since we think that abstract hypotheses and knowledge are generated from observable data via abstraction (as discussed in Figure \ref{fig:hierarchy} later).
\par
We discuss three important perspectives on reasoning and learning. First, knowledge is intrinsically abstract whereas data are intrinsically concrete. The inference of abstraction derives symbolic knowledge from data. The natural view and approach contrast rules of inference and the semantics of Bayesian networks deriving knowledge from another knowledge. Second, our approach looks at how symbolic knowledge can be derived from data. This contrasts the machine learning approach looking at how data can be derived from parameters characterising the data. Third, the inference of abstraction comprises an interpretation and inverse interpretation of formal logic. The inference can be seen as a realisation of top-down and bottom-up processing often used in neuroscience as a metaphor for the information processing of the brain.
\par
The contributions of this paper are summarised as follows. First, this paper results in a new reasoning method that significantly generalises the classical consequence relation for fully data-driven reasoning from consistent and/or possible sources of information. Second, this paper results in a new machine learning method that significantly generalises a sort of the k-nearest neighbour method. The method  empirically outperforms a k-nearest neighbour method in AUC on the MNIST dataset. Third, this paper bridges logic and machine learning in a way that the special cases of the both ends of the bridge are the classical consequence relation and a sort of the K-nearest neighbour method, long-established and long-distance methods in formal logic and machine learning.
\par
This paper is organised as follows. In Section 2, we define a generative reasoning model for inference of abstraction. Section 3 discusses its logical and machine learning correctnesses. We summarise out results in Section 4.
\section{Inference of Abstraction}
Let $\{d_{1},d_{2},...,d_{K}\}$ be a multiset of $K$ data. $D$ denotes a random variable of data whose values are all the elements of $\{d_{1},d_{2},...,d_{K}\}$. For all data $d_{k} (1\leq k\leq K)$, we define the probability of $d_{k}$, denoted by $p(D=d_{k})$, as follows. 
\begin{eqnarray*}
\textstyle{p(D=d_{k})=\frac{1}{K}}
\end{eqnarray*}
\par
$L$ represents a propositional language for simplicity. Let $\{m_{1},m_{2},...,m_{N}\}$ be the set of models of $L$. A model is an assignment of truth values to all the atomic formulas in $L$. Intuitively, each model represents a different state of the world. We assume that each data $d_{k}$ supports a single model. We thus use a function $m$, $\{d_{1},d_{2},...,d_{K}\}\to\{m_{1},m_{2},...,m_{N}\}$, to map each data to the model supported by the data. $M$ denotes a random variable of models whose realisations are all the elements of $\{m_{1},m_{2},...,m_{N}\}$. For all models $m_{n} (1\leq n\leq N)$, we define the probability of $m_{n}$ given $d_{k}$, denoted by $p(M=m_{n}|D=d_{k})$, as follows.
\begin{eqnarray*}
&&p(M=m_{n}|D=d_{k})=
\begin{cases}
1 & \text{if } m_{n}=m(d_{k})\\
0 & \text{otherwise}
\end{cases}
\end{eqnarray*}
\par
The truth value of a propositional formula and first-order closed formula in classical logic is uniquely determined in a state of the world specified by a model of a language. Let $\alpha$ be a formula in $L$. We assume that $\alpha$ is a random variable whose realisations are 0 and 1 meaning false and true respectively. We use symbol $\ms{\alpha}$ to refer to the models of $\alpha$. Namely, $\ms{\alpha=1}$ and $\ms{\alpha=0}$ represent the set of models in which $\alpha$ is true and false, respectively. Let $\mu\in[0,1]$ be a variable, not a random variable. For all formulas $\alpha\in L$, we define the probability of each truth value of $\alpha$ given $m_{n}$, denoted by $p(\alpha|M=m_{n})$, as follows.
\begin{eqnarray*}
&&p(\alpha=1|M=m_{n})=
\begin{cases}
\mu & \text{if } m_{n}\in\llbracket\alpha=1\rrbracket\\
1-\mu & \text{otherwise }
\end{cases}
\\
&&p(\alpha=0|M=m_{n})=
\begin{cases}
\mu & \text{if } m_{n}\in\llbracket\alpha=0\rrbracket\\
1-\mu & \text{otherwise }
\end{cases}
\end{eqnarray*}
The above expressions can be simply written as a Bernoulli distribution with parameter $\mu\in[0,1]$, i.e.,
\begin{eqnarray*}
p(\alpha|M=m_{n})=\mu^{\ms{\alpha}_{m_{n}}}(1-\mu)^{1-\ms{\alpha}_{m_{n}}}.
\end{eqnarray*}
Here, the variable $\mu\in[0,1]$ plays an important role to relate formal logic to machine learning. We will see that $\mu=1$ relates to the classical consequence relation and its generalisation. We also see that $\mu\to 1$ relates to an all-nearest neighbour method, a generalisation of a sort of the K-nearest neighbour method in machine learning. Additionally, $\mu<1$ relates to a smoothed or weighted version of the all-nearest neighbour method. They are all discussed in the next section.
\par
In classical logic, given a model, the truth value of a formula does not change the truth value of another formula. Thus, in probability theory, the truth value of a formula $\alpha_{1}$ is conditionally independent of the truth value of another formula $\alpha_{2}$ given a model $M$, i.e., $p(\alpha_{1}|\alpha_{2},M,D)=p(\alpha_{1}|M,D)$ or equivalently $p(\alpha_{1},\alpha_{2}|M,D)=p(\alpha_{1}|M,D)p(\alpha_{2}|M,D)$. We therefore have
\begin{align}
\textstyle{p(L|M,D)=\prod_{\alpha\in L}p(\alpha|M,D).}\label{eq:1}
\end{align}
Moreover, in classical logic, the truth value of a formula depends on models but not data. Thus, in probability theory, the truth value of a formula $\alpha$ is conditionally independent of data $D$ given a model $M$, i.e., $p(\alpha|M,D)=p(\alpha|M)$. We thus have
\begin{align}
\textstyle{\prod_{\alpha\in L}p(\alpha|M,D)=\prod_{\alpha\in L}p(\alpha|M).}\label{eq:2}
\end{align}
Therefore, the full joint distribution, $p(L, M, D)$, can be written as follows.
\begin{align}
p(L, M, D)&=p(L|M, D)p(M|D)p(D)\nonumber\\
&\textstyle{=\prod_{\alpha\in L}p(\alpha|M)p(M|D)p(D)}\label{eq:3}
\end{align}
Here, the product rule (or chain rule) of probability theory is applied in the first equation, and Equations (\ref{eq:1}) and (\ref{eq:2}) in the second equation. As will be seen later, the joint distribution $p(L, M, D)$ is a probabilistic model of symbolic reasoning from data. We call the joint distribution a generative reasoning model for short. We often represent $p(L, M, D)$ as $p(L, M, D;\mu)$ if our discussion is relevant to $\mu$. We use symbol `$;$' to make sure that $\mu$ is a variable, but not a random variable. In this paper, we assume a finite number of realisations of each random variable.
\par
The full joint distribution implies that we can no longer discuss only the probabilities of individual formulas, but they are derived from data. For example, the probability of $\alpha\in L$ is calculated as follows.
{\small
\begin{align}
p(\alpha)&=\sum_{m}\sum_{d}p(\alpha,m,d)=\sum_{m}p(\alpha|m)\sum_{d}p(m|d)p(d)\label{eq:4}
\end{align}
}
Here, the sum rule of probability theory is applied in the first equation, and Equation (\ref{eq:3}) in the second equation.
\begin{proposition}\label{negation}
Let $p(L,M,D;\mu)$ be a generative reasoning model. For all $\alpha\in L$, $p(\alpha=0)=p(\neg\alpha=1)$ holds.
\end{proposition}
\begin{proof}
For all models $m$, $\alpha$ is false in $m$ if and only if $\lnot\alpha$ is true in $m$. Thus, $\ms{\alpha=0}=\ms{\lnot\alpha=1}$ is the case. Therefore,
\begin{align*}
p(\alpha=0)&\textstyle{=\sum_{m}p(\alpha=0|m)p(m)}\\
&\textstyle{=\sum_{m}\mu^{\ms{\alpha=0}_{m}}(1-\mu)^{1-\ms{\alpha=0}_{m}}p(m)}\\
&\textstyle{=\sum_{m}\mu^{\ms{\lnot\alpha=1}_{m}}(1-\mu)^{1-\ms{\lnot\alpha=1}_{m}}p(m)}\\
&\textstyle{=\sum_{m}p(\lnot\alpha=1|m)p(m)=p(\lnot \alpha=1)}.
\end{align*}
This holds regardless of the value of $\mu$.
\end{proof}
Hereinafter, we replace $\alpha=0$ by $\lnot\alpha=1$ and abbreviate $\lnot\alpha=1$ to $\lnot\alpha$. We also abbreviate $M=m_{n}$ to $m_{n}$ and $D=d_{k}$ to $d_{k}$. 
\par
The hierarchy shown in Figure \ref{fig:hierarchy} illustrates Equation (\ref{eq:4}). The top layer of the hierarchy is a probability distribution over data, the middle layer is a probability distribution over states of the world, often referred to as models in formal logic, and the bottom layer is a probability distribution over a logical formula $\alpha$. A darker colour indicates a higher probability. Each element of a lower layer is an abstraction, i.e., selective ignorance, of the linked element of the upper layer.
\begin{figure}[t]
\begin{center}
 \includegraphics[scale=0.22]{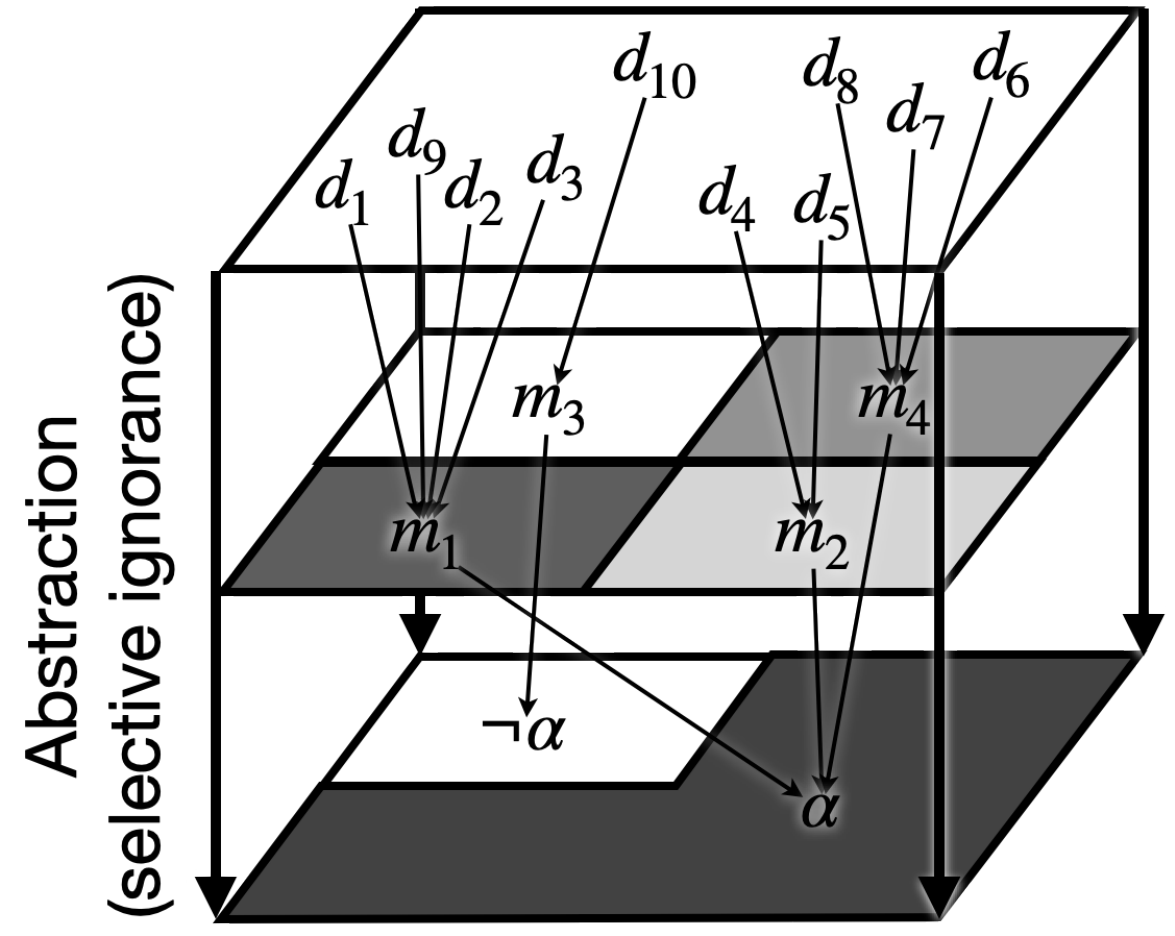}
  \caption{A schematic of how the probability distribution over data determines the probability distribution over logical formulas. For simplicity, an arrow is omitted if the formula at the end of the arrow is false in the model at the start of the arrow and if the model at the end of the arrow is not supported by the data at the start of the arrow.}
  \label{fig:hierarchy}
  \end{center}
\end{figure}
\begin{example}
Let $L$ be a propositional language built with two symbols, $rain$ and $wet$, meaning `rain falls' and `the road gets wet,' respectively. Let $m_{n}(1\leq n\leq 4)$ be the models of $L$ and $d_{k}(1\leq k\leq 10)$ be data about rain and road conditions. Table \ref{tab:hierarchy} shows which data support which models and which models specify which states of the world. The probability of $rain\to wet$ can be calculated using Equation (\ref{eq:4}) as follows.
\begin{align*}
&\textstyle{p(rain\to wet)}\\
&\textstyle{=\sum_{n=1}^{4}p(rain\to wet|m_{n})\sum_{k=1}^{10}p(m_{n}|d_{k})p(d_{k})}\\
&\textstyle{=\mu\sum_{k=1}^{10}p(m_{1}|d_{k})\frac{1}{10}+\mu\sum_{k=1}^{10}p(m_{2}|d_{k})\frac{1}{10}}\\
&\textstyle{~~~~~+(1-\mu)\sum_{k=1}^{10}p(m_{3}|d_{k})\frac{1}{10}+\mu\sum_{k=1}^{10}p(m_{4}|d_{k})\frac{1}{10}}\\
&\textstyle{=\frac{4}{10}\mu+\frac{2}{10}\mu+\frac{1}{10}(1-\mu)+\frac{3}{10}\mu=\frac{1}{10}+\frac{8}{10}\mu}
\end{align*}
Therefore, $p(rain\to wet)=9/10$ when $\mu=1$ or $\mu\to 1$, i.e., $\mu$ approaching 1. Figure \ref{fig:hierarchy} illustrates the calculation and visualises how the probability of $rain\to wet (=\alpha)$ is derived from data.
\end{example}
\begin{table}[t]
\centering
\caption{An example of Figure \ref{fig:hierarchy}. From the left, each column show data, models and the likelihood of the formula.}
\label{tab:hierarchy}
\begin{tabular}{c|ccc|c}
$D$ & $M$ & $rain$ & $wet$  & $p(rain\to wet|M)$\\\hline
$d_{1},d_{2},d_{3},d_{9}$ & $m_{1}$ & 0 & 0 & $\mu$\\
$d_{4},d_{5}$ & $m_{2}$ & 0 & 1 & $\mu$\\
$d_{10}$ & $m_{3}$ & 1 & 0 & $1-\mu$\\
$d_{6},d_{7},d_{8}$ & $m_{4}$ & 1 & 1 & $\mu$
\end{tabular}
\end{table}
\section{Correctness}
Section \ref{sec:logic} discusses the theoretical correctness of the generative reasoning model defined in the previous section in relation to the classical entailment. Section \ref{sec:learning} discusses its empirical correctness in relation to supervised learning. Each section is mostly self-contained.
\subsection{Logic as model-based reasoning}\label{sec:logic}
\subsubsection{Logical reasoning}\label{sec:consistency}
In the previous section, we saw that the probabilities of models and formulas are derived from data. As a result, the probability of a model without support from data and the probability of a formula satisfied only by such models both turn out to be zero. We refer to such models and formulas as being impossible.
\begin{definition}[Possibility]
Let $m$ be a model associated with $L$. $m$ is possible if $p(m)\neq 0$ and impossible otherwise.
\end{definition}
For $\Delta\subseteq L$, we use symbol $\pms{\Delta}$ to denote the set of all the possible models of $\Delta$, i.e., $\pms{\Delta}=\{m\in\ms{\Delta}|p(m)\neq 0\}$. We also use symbol $\pms{\Delta}_{m}$ such that $\pms{\Delta}_{m}=1$ if $m\in\pms{\Delta}$ and $\pms{\Delta}_{m}=0$ otherwise. Obviously, $\pms{\Delta}\subseteq\ms{\Delta}$, for all $\Delta\subseteq L$, and $\pms{\Delta}=\ms{\Delta}$ if all models are possible. If $\Delta$ is inconsistent, $\pms{\Delta}=\ms{\Delta}=\emptyset$. If $\Delta$ is an empty set or if it only includes tautologies then every model satisfies all the formulas in the possibly empty $\Delta$, and thus $\ms{\Delta}$ includes all the models.
\par
In this section, we look at generative reasoning models with $\mu=1$, $p(L,M,D;\mu=1)$, for reasoning from a consistent source of information.  The following theorem relates the probability of a formula to the probability of its models.
\begin{theorem}\label{thrm:consistency}
Let $p(L,M,D;\mu=1)$ be a generative reasoning model, and $\alpha\in L$ and $\Delta\subseteq L$ such that $\ms{\Delta}=\pms{\Delta}$.
\begin{align*}
p(\alpha|\Delta)=
\begin{cases}
\displaystyle{\frac{\sum_{m\in\ms{\Delta}\cap\ms{\alpha}}p(m)}{\sum_{m\in\ms{\Delta}}p(m)}}&\text{if }\ms{\Delta}\neq\emptyset\\
\text{undefined}&\text{otherwise}
\end{cases}
\end{align*}
\end{theorem}
\begin{proof}
Let $|\Delta|$ denote the cardinality of $\Delta$. Dividing models into the ones satisfying all the formulas in $\Delta$ and the others, we have
{\small
\begin{align*}
&p(\alpha|\Delta)=\frac{\sum_{m}p(\alpha|m)p(\Delta|m)p(m)}{\sum_{m}p(\Delta|m)p(m)}\\
&=\frac{\displaystyle{\sum_{m\in\llbracket\Delta\rrbracket}p(m)p(\alpha|m)\mu^{|\Delta|}+\sum_{m\notin\llbracket\Delta\rrbracket}p(m)p(\alpha|m)p(\Delta|m)}}{\displaystyle{\sum_{m\in\llbracket\Delta\rrbracket}p(m)\mu^{|\Delta|}+\sum_{m\notin\llbracket\Delta\rrbracket}p(m)p(\Delta|m)}}.
\end{align*}
}
By definition, $p(\Delta|m)=\prod_{\beta\in\Delta}p(\beta|m)=\prod_{\beta\in\Delta}\mu^{\ms{\beta}_{m}}(1-\mu)^{1-{\ms{\beta}_{m}}}$. For all $m\notin\llbracket\Delta\rrbracket$, there is $\beta\in\Delta$ such that $\ms{\beta}_{m}=0$. Therefore, $p(\Delta|m)=0$ when $\mu=1$, for all $m\notin\llbracket\Delta\rrbracket$. We thus have
{\small
\begin{align*}
p(\alpha|\Delta)=\frac{\displaystyle{\sum_{m\in\llbracket\Delta\rrbracket}p(m)p(\alpha|m)1^{|\Delta|}}}{\displaystyle{\sum_{m\in\llbracket\Delta\rrbracket}p(m)1^{|\Delta|}}}=\frac{\displaystyle{\sum_{m\in\llbracket\Delta\rrbracket}p(m)1^{\ms{\alpha}_{m}}0^{1-\ms{\alpha}_{m}}}}{\displaystyle{\sum_{m\in\llbracket\Delta\rrbracket}p(m)}}.
\end{align*}
}
Since $1^{\ms{\alpha}_{m}}0^{1-\ms{\alpha}_{m}}=1^{1}0^{0}=1$ if $m\in\ms{\alpha}$ and $1^{\ms{\alpha}_{m}}0^{1-\ms{\alpha}_{m}}=1^{0}0^{1}=0$ if $m\notin\ms{\alpha}$, we have
{\small
\begin{align*}
p(\alpha|\Delta)=\frac{\sum_{m\in\llbracket\Delta\rrbracket\cap\ms{\alpha}}p(m)}{\sum_{m\in\llbracket\Delta\rrbracket}p(m)}.
\end{align*}
}
In addition, if $\ms{\Delta}=\emptyset$ then $p(\alpha|\Delta)$ is undefined due to division by zero.
\end{proof}
Recall that a formula $\alpha$ is a logical consequence of a set $\Delta$ of formulas, denoted by $\Delta\cent\alpha$, in classical logic iff (if and only if) $\alpha$ is true in every model in which $\Delta$ is true, i.e., $\ms{\Delta}\subseteq\ms{\alpha}$. The following Corollary shows the relationship between the generative reasoning model $p(L,M,D;\mu=1)$ and the classical consequence relation $\cent$.
\begin{corollary}\label{cor:consistent_reasoning}
Let $p(L,M,D;\mu=1)$ be a generative reasoning model, and $\alpha\in L$ and $\Delta\subseteq L$ such that $\ms{\Delta}=\pms{\Delta}$ and $\ms{\Delta}\neq\emptyset$. $p(\alpha|\Delta)=1$ iff $\Delta\cent\alpha$.
\end{corollary}
\begin{proof}
By the assumptions $\ms{\Delta}=\pms{\Delta}$ and $\ms{\Delta}\neq\emptyset$, $p(m)$ is non zero, for all $m$ in the non-empty set $\ms{\Delta}$. The assumptions thus prohibit a division by zero in Theorem \ref{thrm:consistency}. Therefore, $\frac{\sum_{m\in\ms{\Delta}\cap\ms{\alpha}}p(m)}{\sum_{m\in\ms{\Delta}}p(m)}=1$ iff $\ms{\alpha}\supseteq\ms{\Delta}$, i.e., $\Delta\models\alpha$.
\end{proof}
The following example shows the importance of the assumptions of $\ms{\Delta}=\pms{\Delta}$ and $\ms{\Delta}\neq\emptyset$ in Corollary \ref{cor:consistent_reasoning}.
\begin{table}[t]
\begin{center}
\caption{Some inconsistencies between generative and classical reasoning.}
\label{ex:classical_consequence}
\scalebox{0.8}{
\begin{tabular}{l|l|l}\label{table:consistent}
Generative reasoning & Classical reasoning & Rationale\\\hline
$p(wet|rain,\lnot rain)\neq1$ & $rain,\lnot rain\cent wet$ & $\ms{rain,\lnot rain}=\emptyset$\\
$p(wet|rain)=1$ & $rain\not\cent wet$ & $\ms{rain}\neq\pms{rain}$\\
$p(\lnot rain\lor wet)=1$ & $\not\cent \lnot rain\lor wet$ & $\ms{\emptyset}\neq\pms{\emptyset}$
\end{tabular}
}
\end{center}
\end{table}
\begin{example}
Suppose that the probability distribution in Table \ref{tab:hierarchy} is given by $p(M)=(m_{1},m_{2},m_{3},m_{4})=(0.5,0.2,0,0.3)$. Table \ref{ex:classical_consequence} exemplifies differences between the generative reasoning and classical consequence. The last column explains why the generative reasoning is inconsistent with the classical consequence relation. In particular, the rationale of the last example comes from the fact that Theorem \ref{thrm:consistency} explains $p(\lnot rain\lor wet)$ as $p(\lnot rain\lor wet|\emptyset)$.
\begin{align*}
&p(\lnot rain\lor wet)=p(\lnot rain\lor wet|\emptyset)\\
&=\frac{\sum_{m\in\ms{\emptyset}\cap\ms{\lnot rain\lor wet}}p(m)}{\sum_{m\in\ms{\emptyset}}p(m)}=\frac{\sum_{m\in\ms{\lnot rain\lor wet}}p(m)}{\sum_{m}p(m)}\\
&\textstyle{=\sum_{m\in\ms{\lnot rain\lor wet}}p(m)=1.}
\end{align*}
Here, $\ms{\emptyset}=\{m_{1},m_{2},m_{3},m_{4}\}$ but $\pms{\emptyset}=\{m_{1},m_{2},m_{4}\}$.
\end{example}
\par
Figure \ref{fig:consistent_reasoning} illustrates the assumptions of $\ms{\Delta}=\pms{\Delta}$ and $\ms{\Delta}\neq\emptyset$ for reasoning of $\alpha\in L$ from $\Delta\subseteq L$ using the generative reasoning model $p(L,M,D;\mu=1)$. Both $\alpha$ and $\Delta$ are consistent, i.e., $\ms{\alpha}\neq\emptyset$ and $\ms{\Delta}\neq\emptyset$, since there is at least one model satisfying $\alpha$ and all the formulas in $\Delta$. Such models are highlighted on the middle layer in blue and green, respectively. Figure \ref{fig:consistent_reasoning} also shows that every model satisfying all the formulas in $\Delta$ is possible, i.e., $\ms{\Delta}=\pms{\Delta}$, since there is at least one datum that supports each model of $\Delta$.
\begin{figure}[t]
\begin{center}
 \includegraphics[scale=0.3]{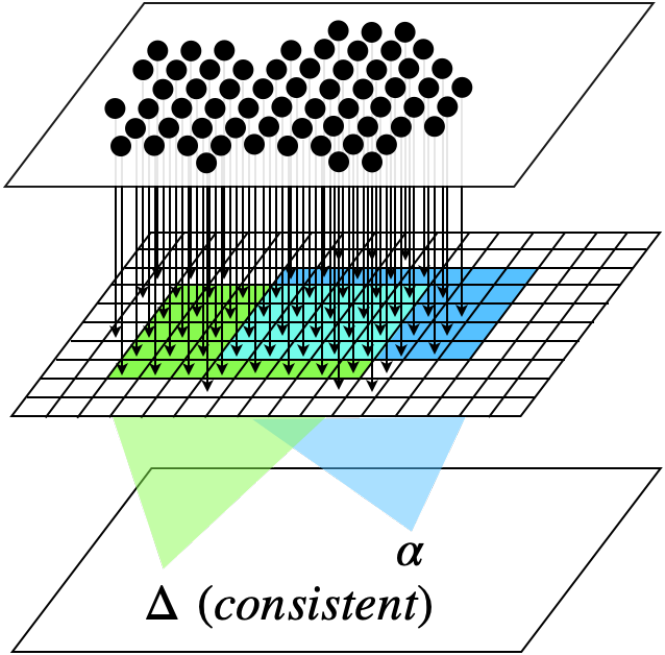}
  \caption{An illustration of the assumptions of $\ms{\Delta}=\pms{\Delta}$ and $\ms{\Delta}\neq\emptyset$ for reasoning of $\alpha\in L$ from $\Delta\subseteq L$ using the generative reasoning model $p(L,M,D;\mu=1)$. Each arrow from a datum to model, denoted respectively by a black circle on the top layer and a cell on the middle layer, represents that the datum supports the model. Each model with an incoming arrow thus has a non-zero probability. A model is coloured in green (resp. blue) if all the formulas in $\Delta$ are (resp. $\alpha$ is) true in the model.}
  \label{fig:consistent_reasoning}
  \end{center}
\end{figure}
\subsubsection{Empirical reasoning}\label{sec:possibility}
Theorem \ref{thrm:consistency} and Corollary \ref{cor:consistent_reasoning} depend on the assumption of $\ms{\Delta}=\pms{\Delta}$. In this section, we cancel the assumption to fully generalise our discussions in Section \ref{sec:consistency}. The following theorem relates the probability of a formula to the probability of its possible models.
\begin{theorem}\label{thrm:possibility}
Let $p(L,M,D;\mu=1)$ be a generative reasoning model, and $\alpha\in L$ and $\Delta\subseteq L$.
\begin{eqnarray*}
p(\alpha|\Delta)=
\begin{cases}
\displaystyle{\frac{\sum_{m\in\pms{\Delta}\cap\pms{\alpha}}p(m)}{\sum_{m\in\pms{\Delta}}p(m)}}&\text{if }\pms{\Delta}\neq\emptyset\\
\text{undefined}&\text{otherwise}
\end{cases}
\end{eqnarray*}
\end{theorem}
\begin{proof}
$p(m)=0$, for all $m\in\ms{\Delta}\setminus\pms{\Delta}$ and $m\in\ms{\alpha}\setminus\pms{\alpha}$. From Theorem \ref{thrm:consistency}, we thus have 
\begin{align*}
\displaystyle{\frac{\sum_{m\in\ms{\Delta}\cap\ms{\alpha}}p(m)}{\sum_{m\in\ms{\Delta}}p(m)}}=\displaystyle{\frac{\sum_{m\in\pms{\Delta}\cap\pms{\alpha}}p(m)}{\sum_{m\in\pms{\Delta}}p(m)}}.
\end{align*}
The condition of $\ms{\Delta}\neq\emptyset$ should be replaced by $\pms{\Delta}\neq\emptyset$, since there is a possibility of $\ms{\Delta}\neq\emptyset$ and $\pms{\Delta}=\emptyset$. Given the condition of $\ms{\Delta}\neq\emptyset$, this causes a probability undefined due to a division by zero.
\end{proof}
In Section \ref{sec:consistency}, we used the classical consequence relation in Corollary \ref{cor:consistent_reasoning} for a logical characterisation of Theorem \ref{thrm:consistency}. In this section, we define an alternative consequence relation for a logical characterisation of Theorem \ref{thrm:possibility}.
\begin{definition}[Empirical consequence]
Let $\Delta\subseteq L$ and $\alpha\in L$. $\alpha$ is an empirical consequence of $\Delta$, denoted by $\Delta\ent\alpha$, if $\pms{\Delta}\subseteq\pms{\alpha}$.
\end{definition}
\begin{proposition}
Let $\Delta\subseteq L$ and $\alpha\in L$. If $\Delta\cent\alpha$ then $\Delta\ent\alpha$, but not vice versa.
\end{proposition}
\begin{proof}
($\Rightarrow$) $\Delta\ent\alpha$ iff $\pms{\Delta}\subseteq\pms{\alpha}$ where $\pms{\Delta}=\{m\in\ms{\Delta}|p(m)\neq0\}$ and $\pms{\alpha}=\{m\in\ms{\alpha}|p(m)\neq0\}$. $\ms{\Delta}\subseteq\ms{\alpha}$ implies $\pms{\Delta}\subseteq\pms{\alpha}$, since $\ms{\Delta}\setminus X\subseteq\ms{\alpha}\setminus X$, for all sets $X$. ($\Leftarrow$) Suppose $\Delta$, $\alpha$ and $m$ such that $\ms{\Delta}=\ms{\alpha}\cup\{m\}$ and $p(m)=0$. Then, $\Delta\ent\alpha$, but $\Delta\not\cent\alpha$.
\end{proof}
The following Corollary shows the relationship between the generative reasoning model $p(L,M,D;\mu=1)$ and the empirical consequence relation $\ent$.
\begin{corollary}\label{cor:possible_reasoning}
Let $p(L,M,D;\mu=1)$ be a generative reasoning model, and $\alpha\in L$ and $\Delta\subseteq L$ such that $\pms{\Delta}\neq\emptyset$. $p(\alpha|\Delta)=1$ iff $\Delta\ent\alpha$.
\end{corollary}
\begin{proof}
$\Delta\ent\alpha$ iff $\pms{\Delta}\subseteq\pms{\alpha}$. $p(m)\neq0$, for all $m\in\pms{\Delta}$. Thus, from Theorem \ref{thrm:possibility}, $p(\alpha|\Delta)=1$ iff $\pms{\Delta}\subseteq\pms{\alpha}$.
\end{proof}
Note that Theorem \ref{thrm:possibility} and Corollary \ref{cor:possible_reasoning} no longer depend on the assumption of $\ms{\Delta}=\pms{\Delta}$ required in Theorem \ref{thrm:consistency} and Corollary \ref{cor:consistent_reasoning}.
\par
Figure \ref{fig:possible_reasoning} illustrates the assumption of $\pms{\Delta}\neq\emptyset$ for reasoning of $\alpha\in L$ from $\Delta\subseteq L$ using the generative reasoning model $p(L, M, D; \mu=1)$. It shows that both $\alpha$ and $\Delta$ are consistent, i.e., $\ms{\alpha}\neq\emptyset$ and $\ms{\Delta}\neq\emptyset$, since there is at least one model for both $\alpha$ and $\Delta$ satisfying the formulas. It also shows that $\Delta$ and $\alpha$ are possible, i.e., $\pms{\Delta}\neq\emptyset$ and $\pms{\alpha}\neq\emptyset$, since there is at least one model for both $\alpha$ and $\Delta$ supported by data.
\begin{figure}[t]
\begin{center}
 \includegraphics[scale=0.3]{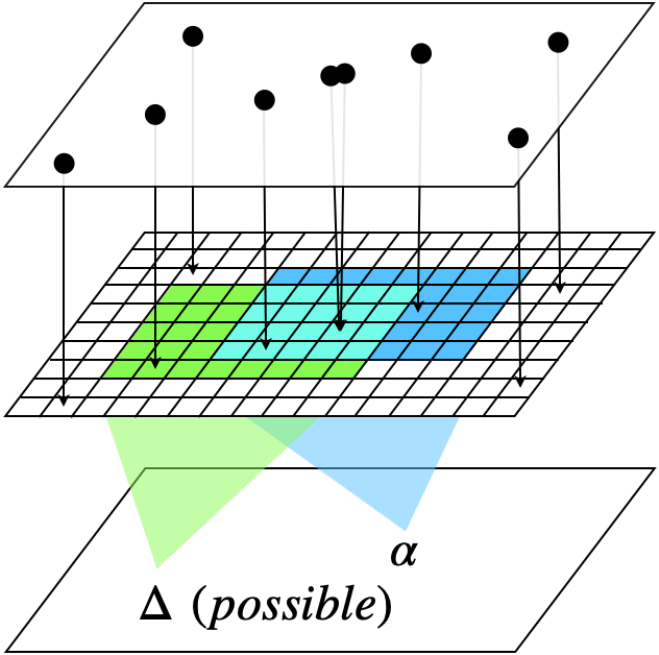}
  \caption{An illustration of the assumption of $\pms{\Delta}\neq\emptyset$ for reasoning of $\alpha\in L$ from $\Delta\subseteq L$ using the generative reasoning model $p(L,M,D;\mu=1)$. The assumption of $\ms{\Delta}=\pms{\Delta}$ assumed in Section \ref{sec:consistency} and illustrated in Figure \ref{fig:consistent_reasoning} is cancelled. It is shown that no data supports some of the models satisfying $\alpha$ and all the formulas in $\Delta$.}
  \label{fig:possible_reasoning}
  \end{center}
\end{figure}
\subsection{Learning as data-based reasoning}\label{sec:learning}
The MNIST dataset contains 70,000 images (60,000 training and 10,000 test images) of handwritten digits from 0 to 9. Each image comprises $28\times 28 (=784)$ pixels in width\texttimes height. Each pixel has a greyscale from 0 to 255 representing pure black and white colours, respectively. We look at two machine learning tasks on MNIST: digit prediction and image generation.
\subsubsection{Digit prediction}
Consider a generative reasoning model $p(L, M, D; \mu)$ where $L$ is built with propositional symbols $digit_{i} (0\leq i\leq 9)$ and $pixel_{j} (1\leq j\leq 28\times 28)$ ($dig_{i}$ and $pix_{j}$ for short), where $digit_{i}$ represents that an image is of digit $i$, and $pixel_{j}$ that the greyscale of the $j$th pixel of an image is above the threshold of 30. 
All the ten digit variables and 28\texttimes 28 pixel variables can take the two states, true of false. $L$ thus has $2^{10+28\times 28}$ models in total, and each of the models is a value of the random variable $M$. Each training image is a value of the random variable $D$. We use the following fact in the machine learning context.
\begin{proposition}\label{prop:data-basedCP}
Let $p(L$, $M$, $D; \mu\in(0.5,1))$ be a generative reasoning model, and $\alpha\in L$ and $\Delta\subseteq L$.
{\small
\begin{align*}
p(\alpha|\Delta)=\frac{\sum_{d}p(\alpha|d)\prod_{\beta\in\Delta}p(\beta|d)}{\sum_{d}\prod_{\beta\in\Delta}p(\beta|d)}
\end{align*}
}
\end{proposition}
\begin{proof}
For all $\gamma\in L$ and data $d$, we have
{\small 
\begin{align*}
&p(\gamma|d)=\frac{\sum_{m}p(\gamma,m,d)}{p(d)}=\frac{\sum_{m}p(\gamma|m)p(m|d)p(d)}{p(d)}\\
&=\frac{p(\gamma|m(d))\cancel{p(d)}}{\cancel{p(d)}}=p(\gamma|m(d))=\mu^{\ms{\gamma}_{m(d)}}(1-\mu)^{\ms{\gamma}_{m(d)}}.
\end{align*}
}
Since $\mu\notin\{0,1\}$, $p(\gamma|d)\neq 0$. We also have
{\small
\begin{align*}
&p(\alpha|\Delta)=\frac{\sum_{d}\sum_{m}p(\alpha|m)\prod_{\beta\in\Delta}p(\beta|m)p(m|d)p(d)}{\sum_{d}\sum_{m}\prod_{\beta\in\Delta}p(\beta|m)p(m|d)p(d)}\\
&=\frac{\displaystyle{\sum_{d}p(\alpha|m(d))\prod_{\beta\in\Delta}p(\beta|m(d))}}{\displaystyle{\sum_{d}\prod_{\beta\in\Delta}p(\beta|m(d))}}=\frac{\displaystyle{\sum_{d}p(\alpha|d)\prod_{\beta\in\Delta}p(\beta|d)}}{\displaystyle{\sum_{d}\prod_{\beta\in\Delta}p(\beta|d)}}.
\end{align*}
}
Since $\mu\notin\{0,1\}$, this does not cause division by zero.
\end{proof}
\par
For digit prediction, we first look at the generative reasoning model $p(L$, $M$, $D;\mu\to1)$ where $\mu\to 1$ represents that $\mu$ approaches one, i.e., $\lim_{\mu\to 1}$. Given all the 60k training images, we use the following instance of Proposition \ref{prop:data-basedCP}.
\begin{align}\label{eq:mnist1}
&p(Digit_{i}|Pixel_{1},...,Pixel_{28\times 28})\nonumber\\
&=\frac{\sum_{k=1}^{60k}p(Digit_{i}|d_{k})\prod_{j=1}^{28\times 28}p(Pixel_{j}|d_{k})}{\sum_{k=1}^{60k}\prod_{j=1}^{28\times 28}p(Pixel_{j}|d_{k})}
\end{align}
Here, we capitalised the propositional symbols so that it is clear that they are not formulas being true, e.g., $digit_{i}=1$, but random variables without observed values.
\begin{figure}[t]
\centering
 \includegraphics[scale=0.24]{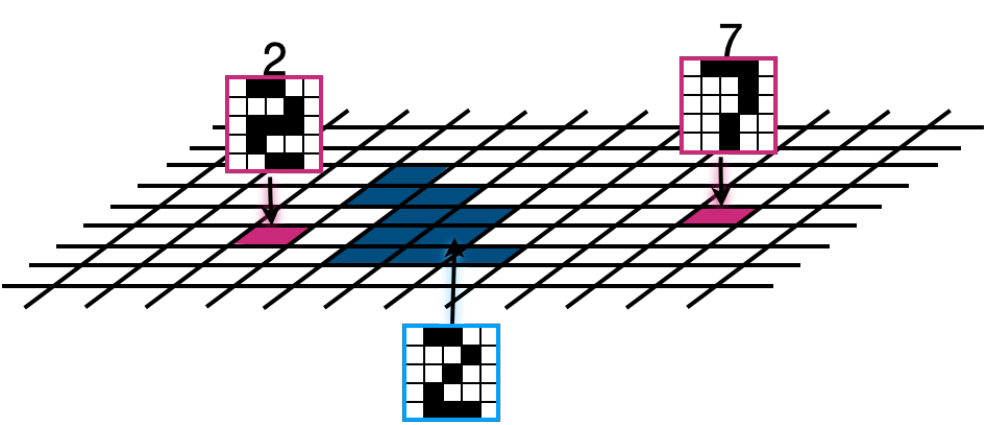}
 \includegraphics[scale=0.24]{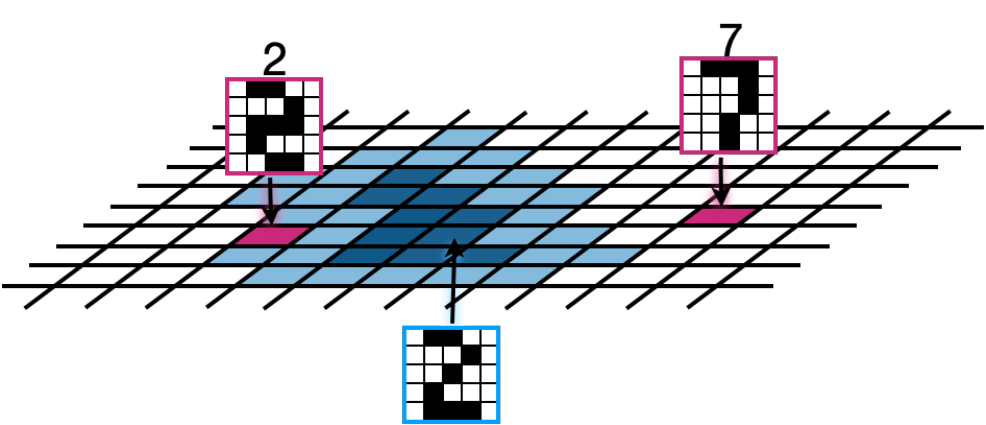}
 \\
 \includegraphics[scale=0.24]{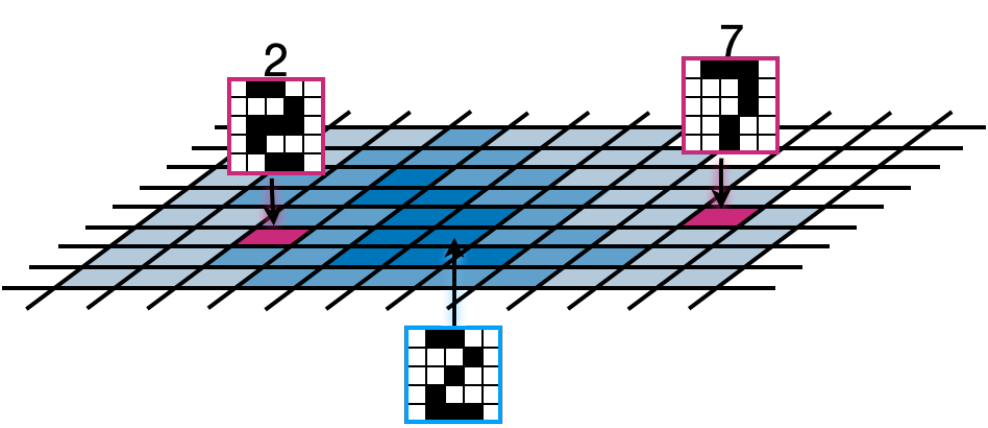}
 \includegraphics[scale=0.24]{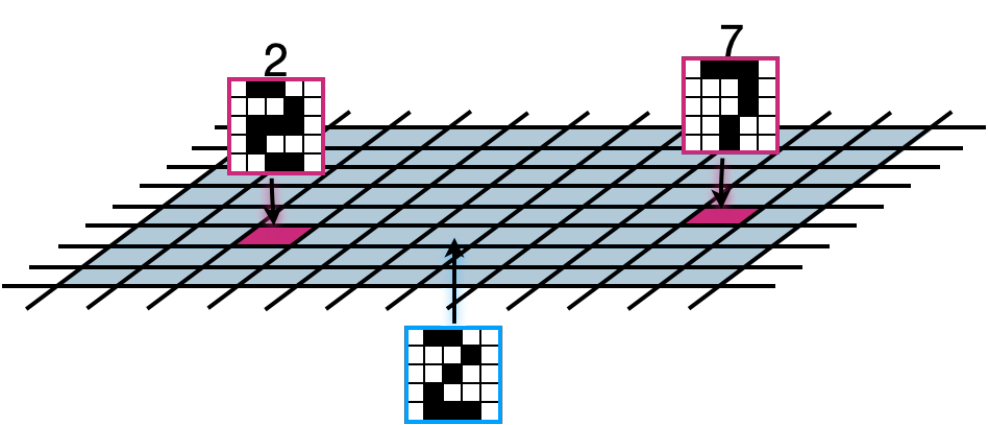}
\caption{Each cell of the grid is a model of $L$. The training and test images are shown above and below the grid, respectively. The blue cells on the top-left grid show that the prediction fails with $\mu=1$, since no training image is found in the models of the test image. The light blue cells on the top-right grid show that the prediction succeeds with $\mu\to 1$, since the limit expands the models of the test image until its best matched training image is found. The bottom left and right grids illustrate $\mu\in(0.5,1)$ and $\mu=0.5$, respectively.}
\label{fig:DigitPrediction_illustration}
\end{figure}
\begin{example}[Digit prediction with $p(L,M,D;\mu\to1)$]\label{ex:dp1}
Let $L$ be built with propositional symbols $digit_{i} (0\leq i\leq 9)$ and $pixel_{j} (1\leq j\leq 5\times 5)$. Let the following two 5\texttimes 5-pixel images with the purple borders be training images and the following one 5\texttimes 5-pixel image with the blue border be a test image.
\\
\centerline{$\imgeqL{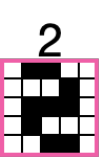}~~~~~~\imgeqL{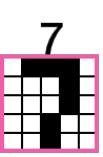}~~~~~~\imgeqL{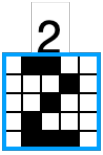}$}
\par
The label of each image is the digit of the image. Each 5\texttimes 5-pixel training image with its digit instantiates the random variable $D$. Equation (\ref{eq:mnist1}) can then be instantiated as follows, where $\bm{Pixel}=(Pixel_{1},...,Pixel_{5\times 5})$.
{\small
\begin{align}\label{eq:ToyMNIST_prediction1}
&p(Dig_{i}|\bm{Pixel})=\frac{\sum_{k=1}^{2}p(Dig_{i}|d_{k})\prod_{j=1}^{5\times 5}p(Pix_{j}|d_{k})}{\sum_{k=1}^{2}\prod_{j=1}^{5\times 5}p(Pix_{j}|d_{k})}\\
&=\frac{\displaystyle{p(Dig_{i}|\imgeq{figs/train1_2.pdf})\prod_{j=1}^{5\times 5}p(Pix_{j}|\imgeq{figs/train1_2.pdf})+p(Dig_{i}|\imgeq{figs/train2_7.pdf})\prod_{j=1}^{5\times 5}p(Pix_{j}|\imgeq{figs/train2_7.pdf})}}{\displaystyle{\prod_{j=1}^{5\times 5}p(Pix_{j}|\imgeq{figs/train1_2.pdf})+\prod_{j=1}^{5\times 5}p(Pix_{j}|\imgeq{figs/train2_7.pdf})}}\nonumber
\end{align}
}
The map $m$ from each training image to a model of $L$ is obvious. We have the following likelihoods, where $j$ indexes pixels from left to right and top to bottom.
{\small
\begin{align*}
&p(Dig_{i}\!\!=\!\!1|\imgeq{figs/train1_2.pdf})\!\!=\!\!
\begin{cases}
\mu~~\text{ if } i=2\\
1-\mu\text{ o/w}
\end{cases}
p(Dig_{i}\!\!=\!\!1|\imgeq{figs/train2_7.pdf})\!\!=\!\!
\begin{cases}
\mu~~\text{ if } i=7\\
1-\mu\text{ o/w}
\end{cases}
\\
&p(Pix_{j}\!\!=\!\!1|\imgeq{figs/train1_2.pdf})\!\!=\!\!
\begin{cases}\nonumber
\mu~\text{if }j\in\{1,4\text{--}8,10,11,15,16,18\text{--}22,25\}\\
1-\mu~\text{otherwise}
\end{cases}
\\
&p(Pix_{j}\!\!=\!\!1|\imgeq{figs/train2_7.pdf})\!\!=\!\!
\begin{cases}\nonumber
\mu~\text{if }j\in\{1,5\text{--}8,10\text{--}13,15\text{--}17,19\text{--}22,24,25\}\\
1-\mu~\text{otherwise}
\end{cases}
\end{align*}
}
Let $pixel_{j}$ be 1 if the $j$th pixel value of the test image is above the threshold of thirty, and 0 otherwise. We have
\begin{align*}
pixel_{j}=
\begin{cases}
1 & \text{if } j\in\{1, 4\text{--}8, 10\text{--}12, 14\text{--}16, 18\text{--}21, 25\}\\
0 & \text{otherwise, i.e., } j\in\{2, 3, 9, 13, 17, 22\text{--}24\}.
\end{cases}
\end{align*}
Let $\bm{pixel}$, abbreviated to $\bm{pix}$, denote $(Pixel_{1}=pixel_{1}$, $Pixel_{2}=pixel_{2}$, ..., $Pixel_{5\times 5}=pixel_{5\times 5})$. Equation (\ref{eq:ToyMNIST_prediction1}) can then be instantiated as follows.
{\small
\begin{align*}
p(Dig_{i}=1|\bm{pix})=\frac{p(Dig_{i}=1|\imgeq{figs/train1_2.pdf})X_{1}+p(Dig_{i}=1|\imgeq{figs/train2_7.pdf})X_{2}}{X_{1}+X_{2}}
\end{align*}
\begin{numcases}{=}
\frac{\mu^{23}(1-\mu)^{3}+\mu^{18}(1-\mu)^{8}}{\mu^{22}(1-\mu)^{3}+\mu^{18}(1-\mu)^{7}} & \text{if $i=2$}\label{eq:ToyMNIST_prediction1_11}\\
\frac{\mu^{22}(1-\mu)^{4}+\mu^{19}(1-\mu)^{7}}{\mu^{22}(1-\mu)^{3}+\mu^{18}(1-\mu)^{7}} & \text{if $i=7$}\label{eq:ToyMNIST_prediction1_12}\\
\frac{\mu^{22}(1-\mu)^{4}+\mu^{18}(1-\mu)^{8}}{\mu^{22}(1-\mu)^{3}+\mu^{18}(1-\mu)^{7}} & \text{otherwise}\label{eq:ToyMNIST_prediction1_13}
\end{numcases}
}
Here, $X_{1}$ and $X_{2}$ were calculated as follows.
{\small
\begin{align*}
&\textstyle{X_{1}=\prod_{j=1}^{5\times 5}p(Pixel_{j}=pixel_{j}|\imgeq{figs/train1_2.pdf})=\mu^{22}(1-\mu)^{3}}\\
&\textstyle{X_{2}=\prod_{j=1}^{5\times 5}p(Pixel_{j}=pixel_{j}|\imgeq{figs/train2_7.pdf})=\mu^{18}(1-\mu)^{7}}
\end{align*}
}
Given $\mu\to 1$, Equations (\ref{eq:ToyMNIST_prediction1_11}), (\ref{eq:ToyMNIST_prediction1_12}) and (\ref{eq:ToyMNIST_prediction1_13}) thus turn out to be
{\small
\begin{numcases}{}
\lim_{\mu\to 1}\frac{\mu^{5}+(1-\mu)^{5}}{\mu^{4}+(1-\mu)^{4}}=\frac{1}{1}=1 & \text{ if $i=2$}\label{eq:ToyMNIST_prediction1_21}\\
\lim_{\mu\to 1}\frac{\mu^{4}(1-\mu)+\mu(1-\mu)^{4}}{\mu^{4}+(1-\mu)^{4}}=\frac{0}{1}=0 & \text{ if $i=7$}\label{eq:ToyMNIST_prediction1_22}\\
\lim_{\mu\to 1}\frac{\mu^{4}(1-\mu)+(1-\mu)^{5}}{\mu^{4}+(1-\mu)^{4}}=\frac{0}{1}=0 & \text{ otherwise}.\label{eq:ToyMNIST_prediction1_23}
\end{numcases}
}
Figure \ref{fig:DigitPrediction_illustration} illustrates the digit prediction with different $\mu$ values. It shows a reasonable role of the limit used in Equations (\ref{eq:ToyMNIST_prediction1_21}), (\ref{eq:ToyMNIST_prediction1_22}) and (\ref{eq:ToyMNIST_prediction1_23}). The limit allows us to cancel out $(1-\mu)^{3}$ from the equations. Here, $(1-\mu)$ represents a mismatch between the test image and the training image, and thus, $(1-\mu)^{3}$ represents a mismatch between the test image and the training image with the best match for the test image. The limit thus subtracts the mismatch from all the training images. As a result, the digit of the given image turns out to be the digit of its best matched training image.
\end{example}
\par
As shown in Equations (\ref{eq:ToyMNIST_prediction1_21}), (\ref{eq:ToyMNIST_prediction1_22}) and (\ref{eq:ToyMNIST_prediction1_23}), the denominator turns out to be the number of training images whose pixel values are maximally the same as $Pixel_{1},...,Pixel_{28\times 28}$, the pixel values of a test image. Amongst them, the numerator turns out to be the number of training images whose digit values are the same as $Digit_{i}$, the digit value of the test image. As a result, the above conditional probability can be seen as an all-nearest neighbours method, which generalises the k-nearest neighbours (kNN) method classifying test data by a majority vote from the k nearest training data. This is a reasonable solution to a well-known problem that it is often difficult to settle an appropriate value of k for kNN methods. Moreover, the search for the nearest neighbours and the use of them in prediction are given a unified computational account by Equation (\ref{eq:mnist1}).
\par
In the machine learning context, we until now saw generative reasoning models $p(L,M,D;\mu\to 1)$ as a sort of an all-nearest neighbours method. We will next see generative reasoning models $p(L,M,D;\mu\in(0.5,1))$ as a smoothed or weighted version of the all-nearest neighbours method.
\begin{figure}[t]
\centering
 \includegraphics[scale=0.24]{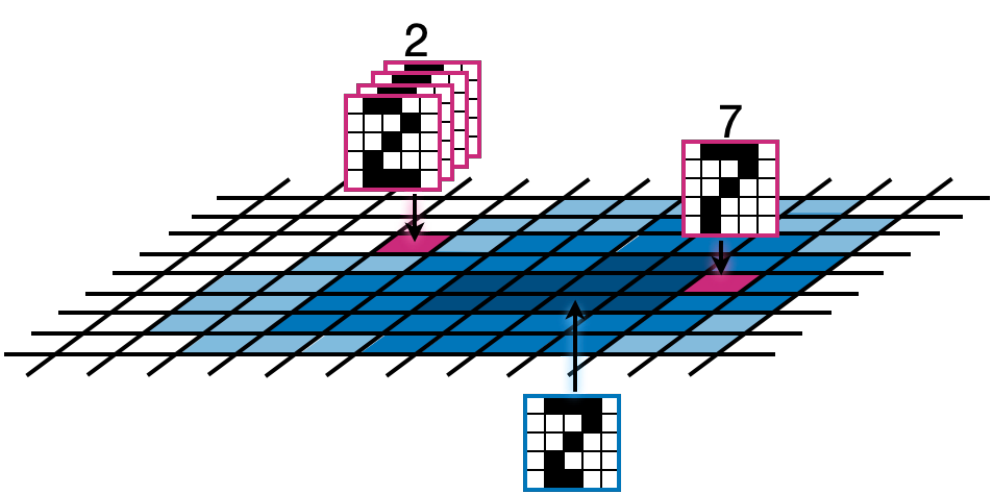}
 \includegraphics[scale=0.23]{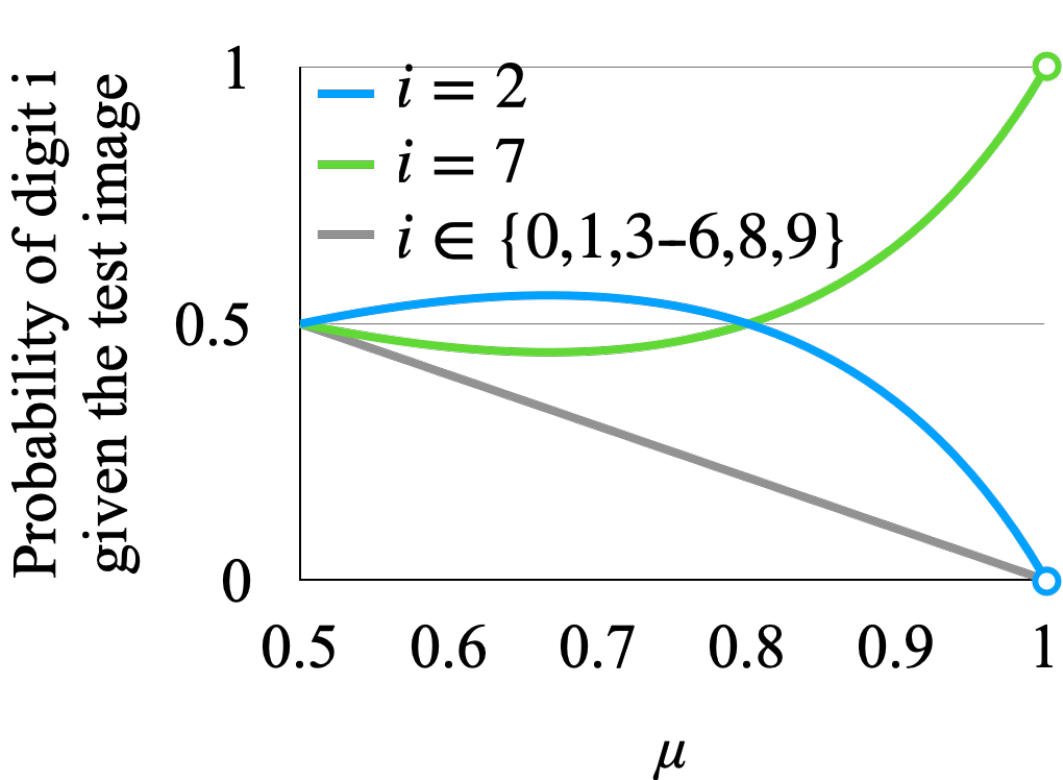}
\caption{The prediction fails with $\mu\to 1$, since the test image and its nearest training image have different digits (see the medium blue cells). It can succeed with $\mu\in(0.5,1)$, since the models of the test image is expanded beyond its nearest training image for its second and further nearest training images (see the light blue cells). The curves on the right show the values of Expressions (\ref{eq:prediction2_1}), (\ref{eq:prediction2_2}) and (\ref{eq:prediction2_3}).}
\label{fig:mnistToy_muto1_3}
\end{figure}
\begin{example}[Digit prediction with $p(L,M,D;\mu\in(0.5,1))$ (Continued)]
Consider the following five 5\texttimes 5-pixel training images and one 5\texttimes 5-pixel test image with the labels of their digits.
\\
\centerline{$\imgeqL{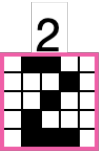}~~~~~~\imgeqL{figs/train3_2.pdf}~~~~~~\imgeqL{figs/train3_2.pdf}~~~~~~\imgeqL{figs/train3_2.pdf}~~~~~~\imgeqL{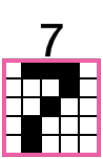}~~~~~~\imgeqL{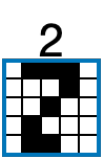}$}
\par
Going through the same process we discussed in Example \ref{ex:dp1}, we can now instantiate Equation (\ref{eq:mnist1}) as follows.
{\small
\begin{numcases}{p(Dig_{i}=1|\bm{pix})=}
\frac{5\mu(1-\mu)}{4(1-\mu)+\mu} & \text{if } i=2 \label{eq:prediction2_1}\\
\frac{4(1-\mu)^{2}+\mu^{2}}{4(1-\mu)+\mu} & \text{if } i=7 \label{eq:prediction2_2}\\
\frac{4(1-\mu)^{2}+\mu(1-\mu)}{4(1-\mu)+\mu} & \text{otherwise} \label{eq:prediction2_3}
\end{numcases}
}
Given $\mu\to 1$, each equation turns out to be 0, 1 and 0, respectively, which are all reasonable as the test image and its best matched training image have different digits. However, given $\mu\in(0.5,0.8)$, the probability of the digit being two is equal or larger than the probability of the digit being seven (see the curves in Figure \ref{fig:mnistToy_muto1_3}). This is also reasonable as the test image and all of the relatively large number of its second matched training images have the same digits. Here, the qualitative effect of the single best match for the test image is suppressed by the quantitative effect of the multiple second match. As shown in Figure \ref{fig:mnistToy_muto1_3}, $\mu$ functions to balance the effects of matching quality and quantity.
\end{example}
Figure \ref{fig:LCs} shows the learning curves generated by Equation (\ref{eq:mnist1}) using the real MNIST dataset. The baseline is given by the kNN method with different k values. We use AUC, the area under ROC (receiver operating characteristic) curve, for performance evaluation, since the generative reasoning model returns probabilistic outputs. $\mu\to 1$ experiences overfitting, since the number of the training images best matched for each test image is relatively too small to discard anomalies. This is similar to the 1NN method where only one nearest neighbour training image is used in prediction.
\begin{figure}[t]
\begin{center}
\includegraphics[scale=0.3]{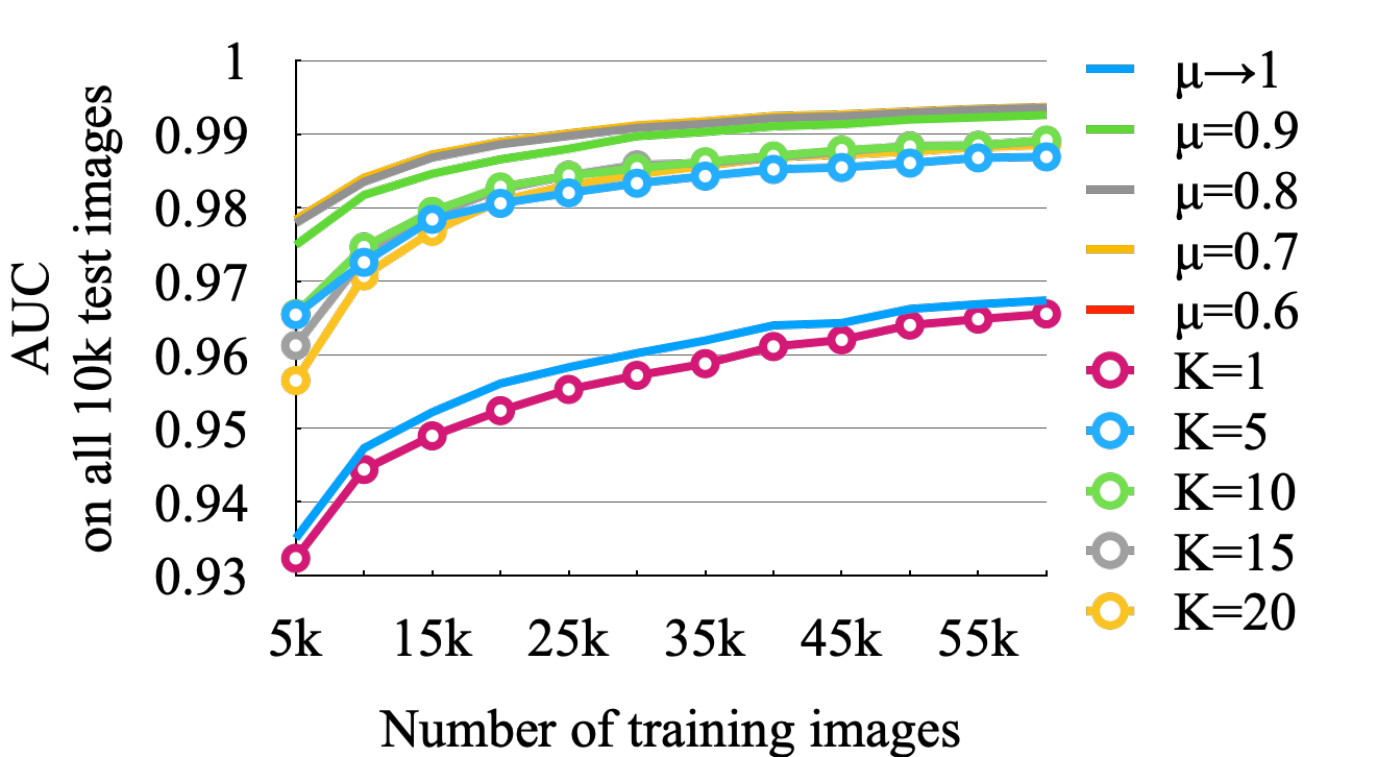}
\end{center}
\caption{The learning curves of the generative reasoning model with different $\mu$ values. The baseline is given by the kNN method built using the `KNeighborsClassifier' function \citep{scikit-learn} with default setting, i.e., the `uniform' weights and `auto' algorithm. The training images were extracted from the beginning.}
\label{fig:LCs}
\end{figure}
\subsubsection{Image generation}

\begin{figure}[t]
\begin{center}
\includegraphics[scale=0.06]{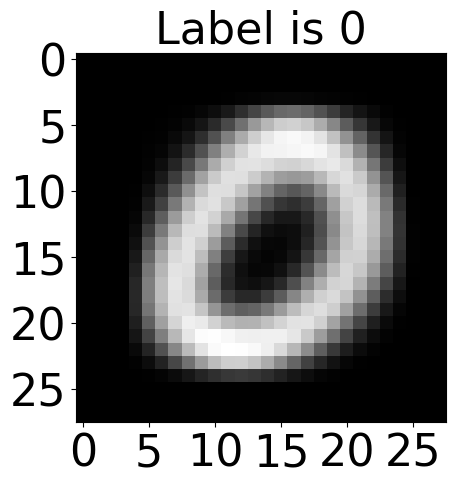}
\includegraphics[scale=0.06]{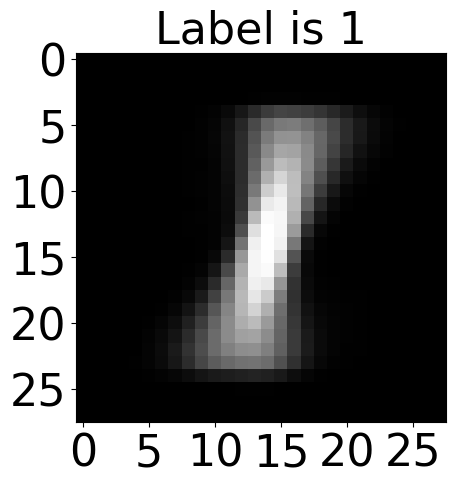}
\includegraphics[scale=0.06]{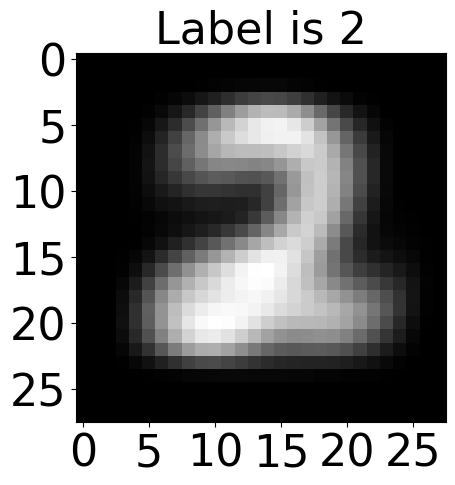}
\includegraphics[scale=0.06]{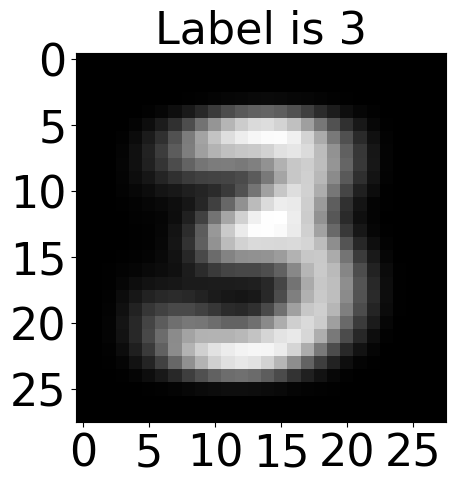}
\includegraphics[scale=0.06]{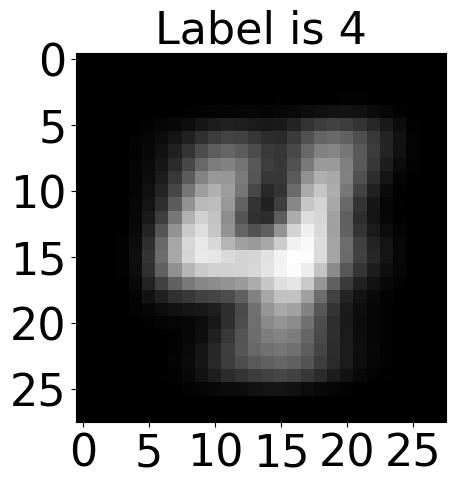}
\includegraphics[scale=0.06]{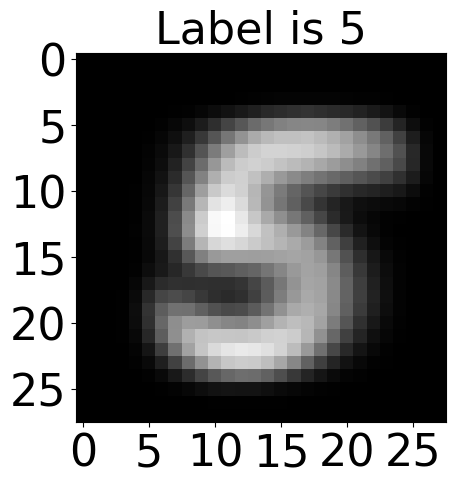}
\includegraphics[scale=0.06]{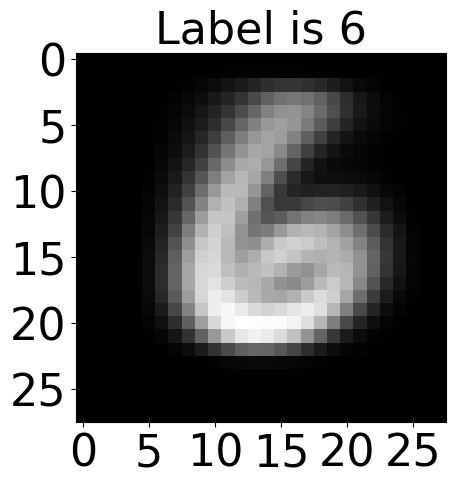}
\includegraphics[scale=0.06]{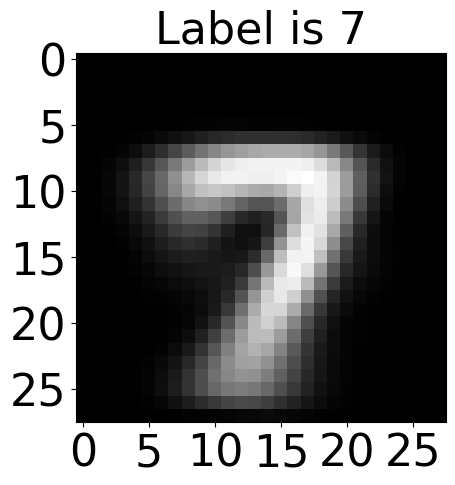}
\includegraphics[scale=0.06]{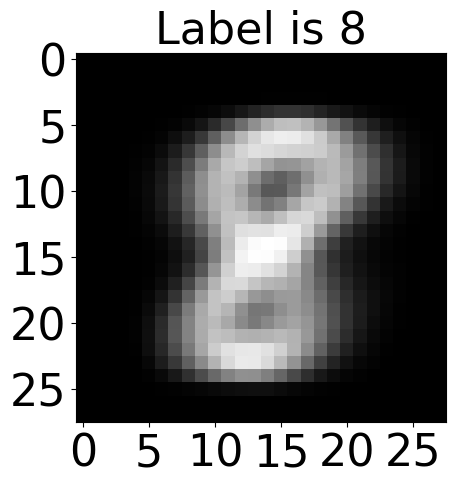}
\includegraphics[scale=0.06]{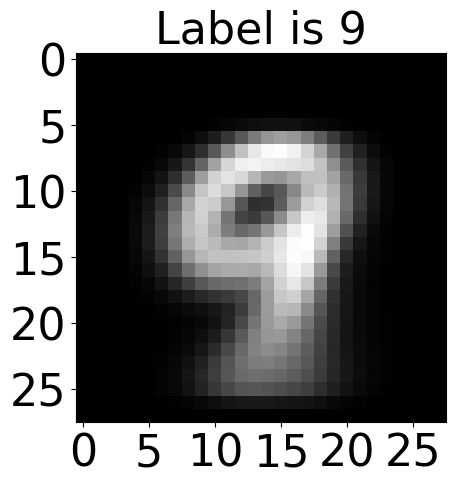}
\caption{The images of all the ten digits. We normalised $p(Pixel_{j}|Digit_{i})\in[0,1]$ to the grayscale between 0 (black) and 255 (white), for all pixels $j$ ($1\leq j\leq 28\times 28$).}
\label{fig:generation}
\end{center}
\begin{center}
\includegraphics[scale=0.06]{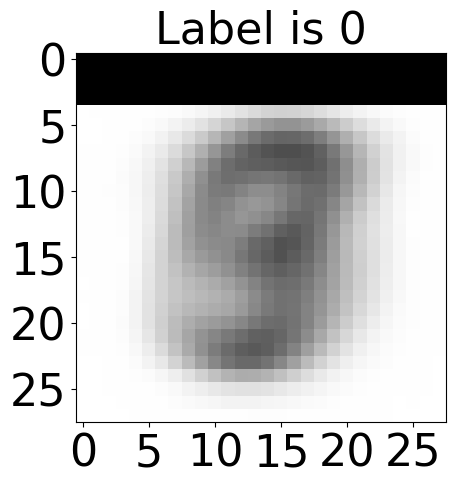}
\includegraphics[scale=0.06]{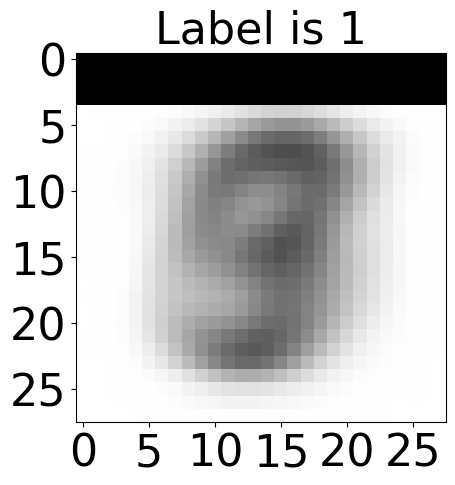}
\includegraphics[scale=0.06]{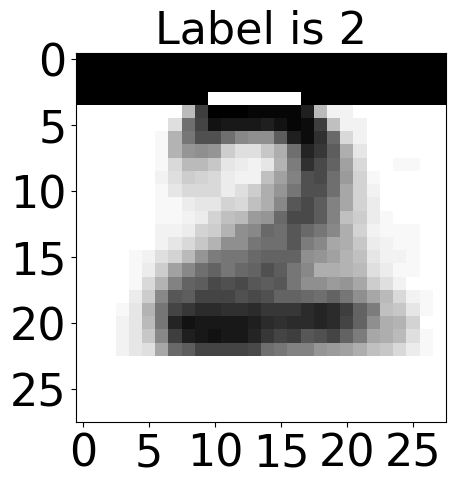}
\includegraphics[scale=0.06]{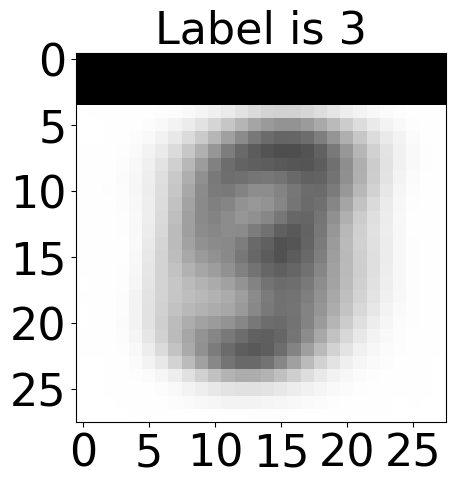}
\includegraphics[scale=0.06]{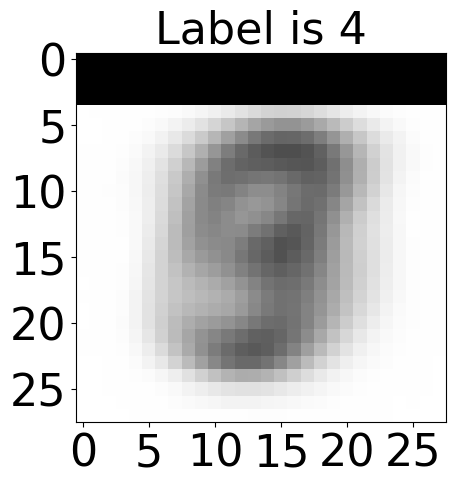}
\includegraphics[scale=0.06]{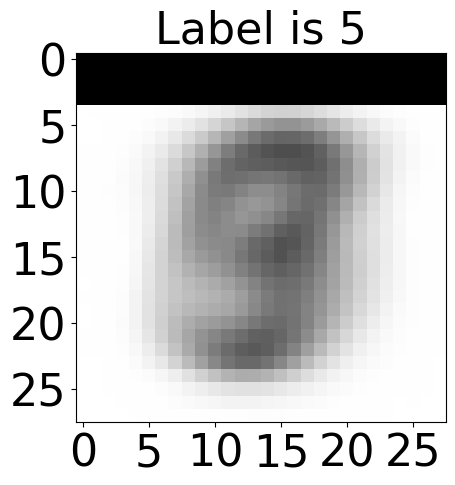}
\includegraphics[scale=0.06]{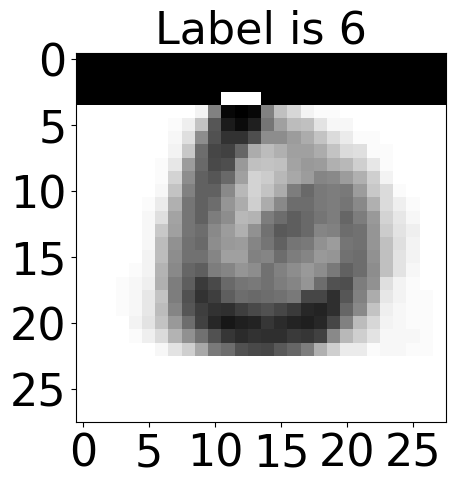}
\includegraphics[scale=0.06]{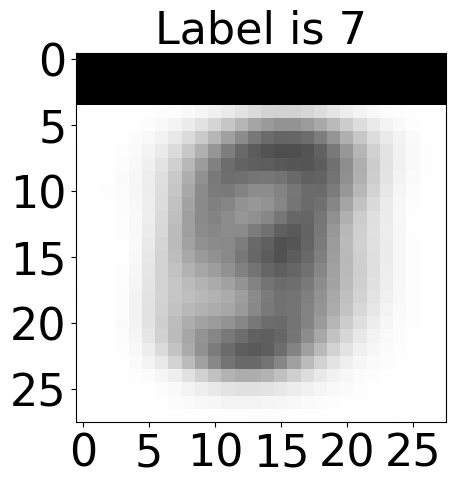}
\includegraphics[scale=0.06]{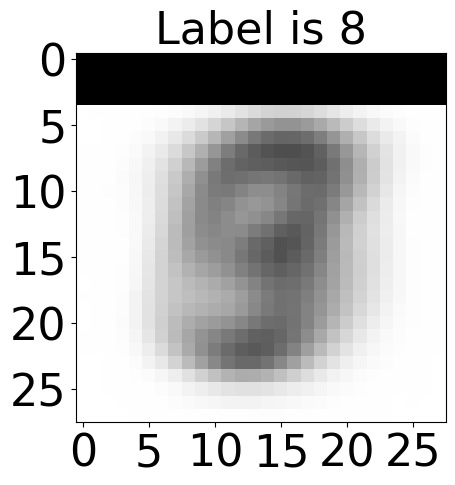}
\includegraphics[scale=0.06]{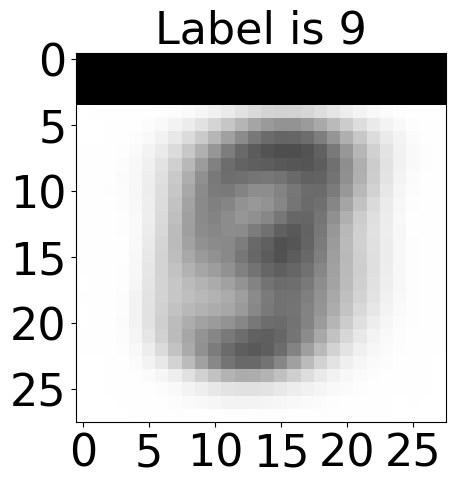}
\\
\includegraphics[scale=0.06]{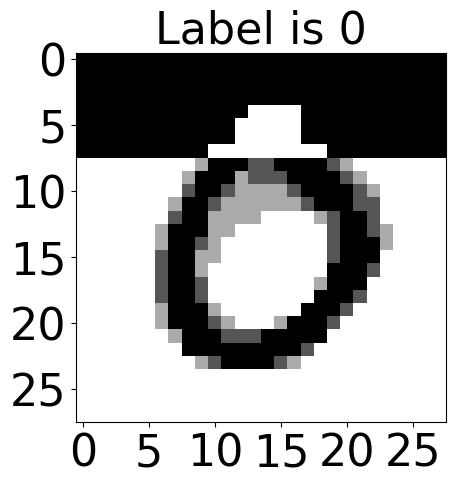}
\includegraphics[scale=0.06]{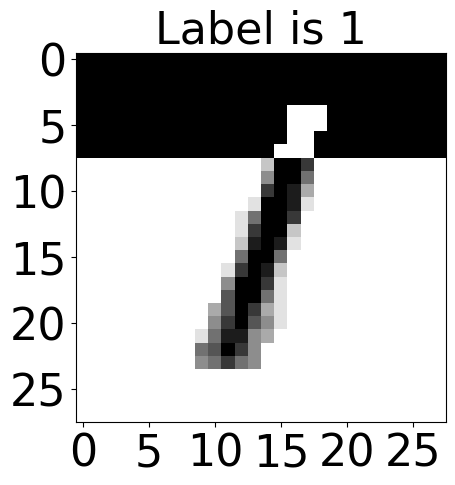}
\includegraphics[scale=0.06]{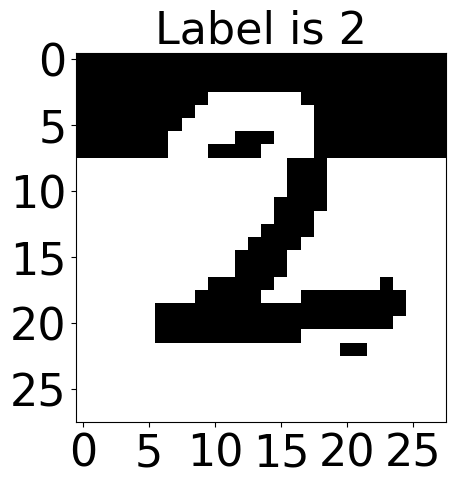}
\includegraphics[scale=0.06]{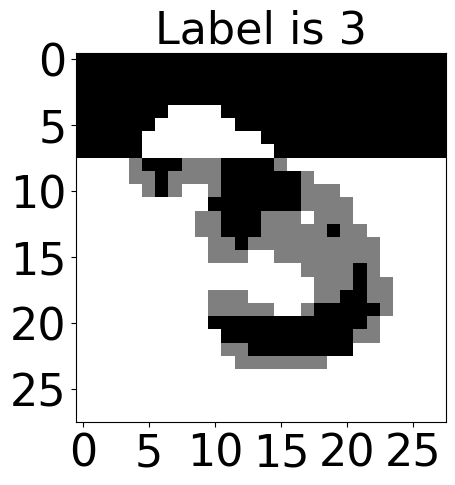}
\includegraphics[scale=0.06]{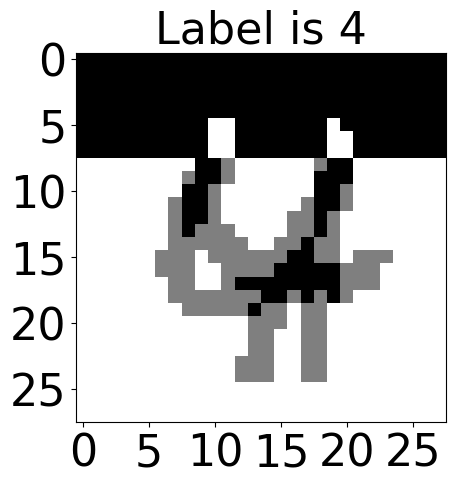}
\includegraphics[scale=0.06]{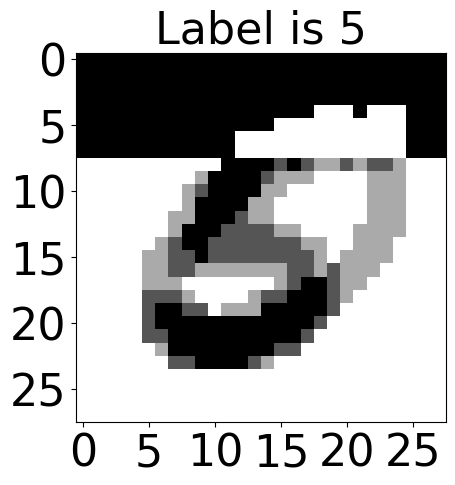}
\includegraphics[scale=0.06]{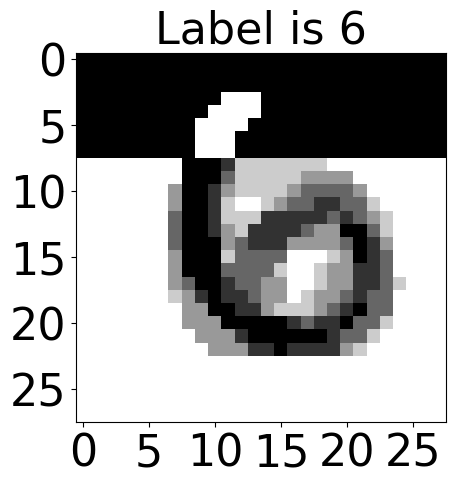}
\includegraphics[scale=0.06]{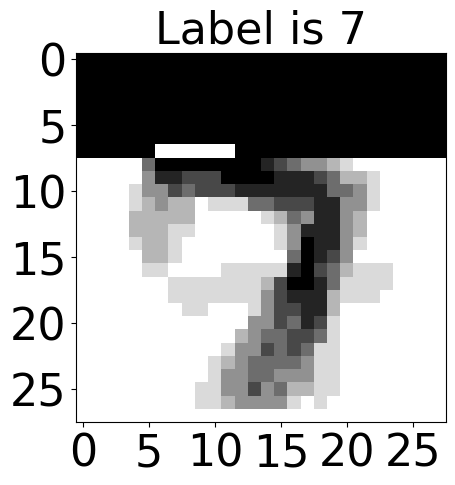}
\includegraphics[scale=0.06]{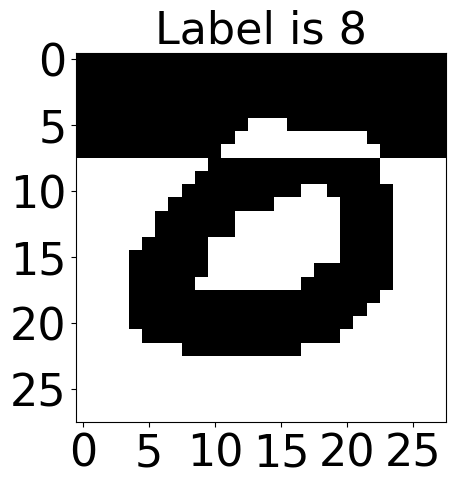}
\includegraphics[scale=0.06]{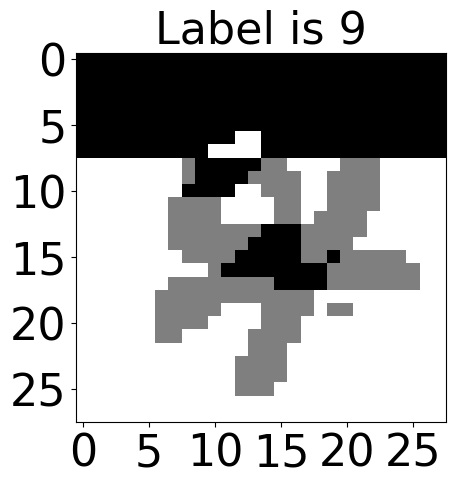}
\\
\includegraphics[scale=0.06]{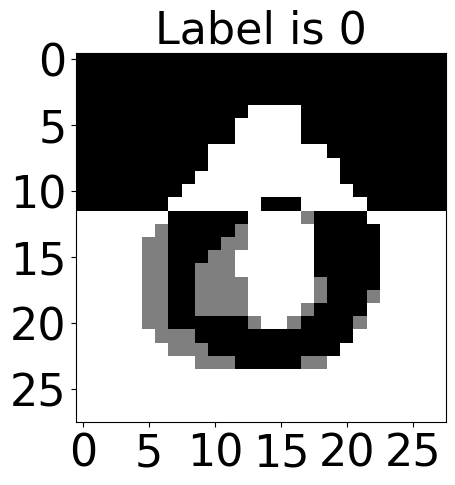}
\includegraphics[scale=0.06]{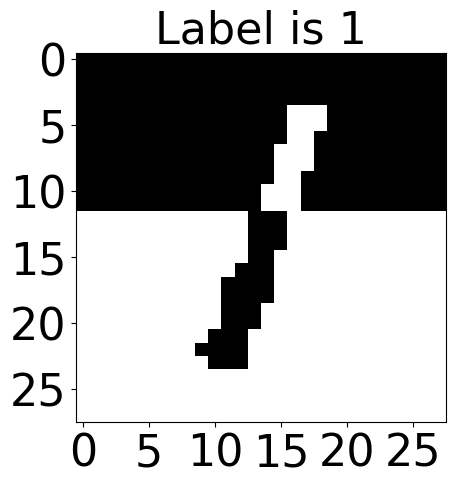}
\includegraphics[scale=0.06]{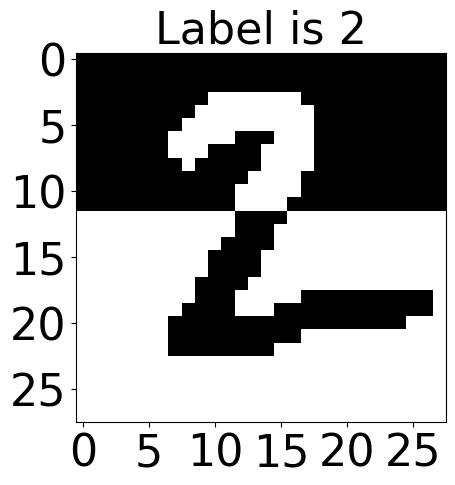}
\includegraphics[scale=0.06]{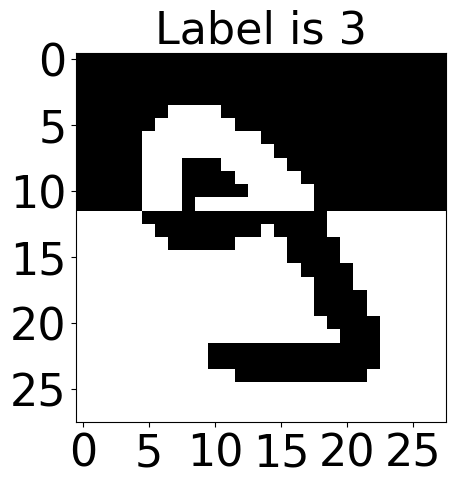}
\includegraphics[scale=0.06]{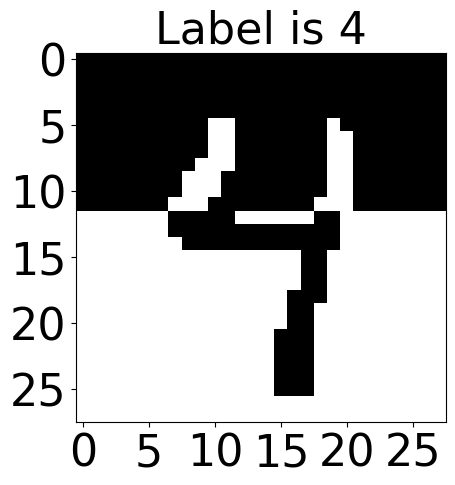}
\includegraphics[scale=0.06]{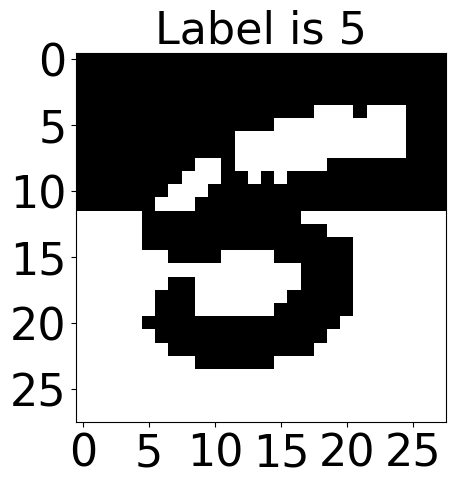}
\includegraphics[scale=0.06]{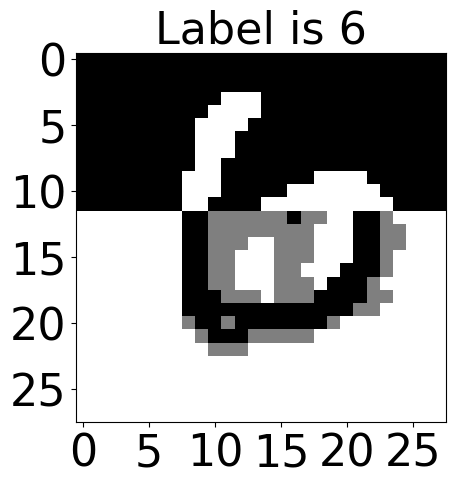}
\includegraphics[scale=0.06]{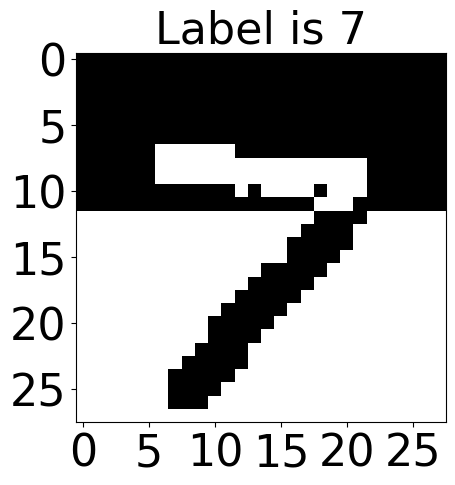}
\includegraphics[scale=0.06]{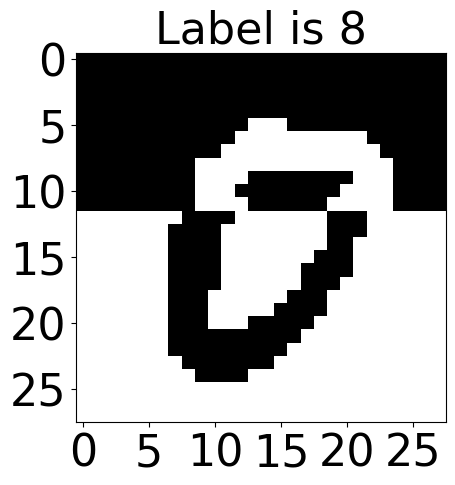}
\includegraphics[scale=0.06]{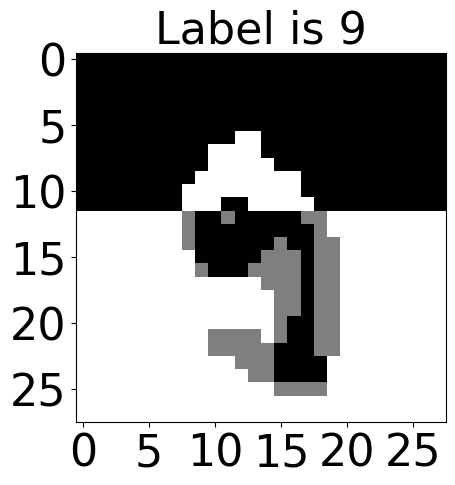}
\\
\includegraphics[scale=0.06]{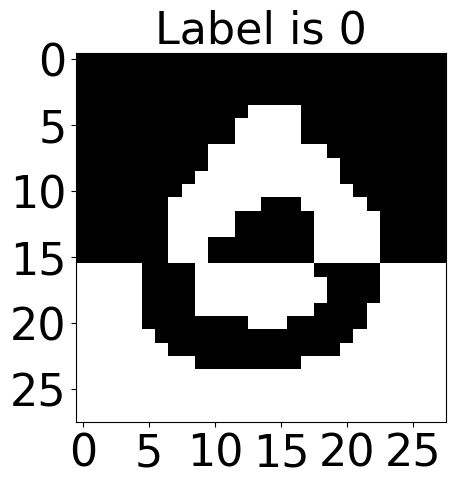}
\includegraphics[scale=0.06]{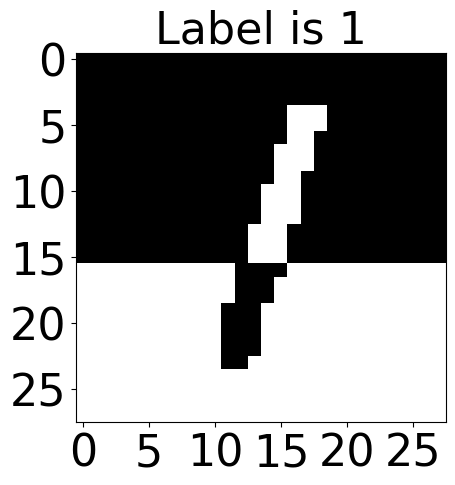}
\includegraphics[scale=0.06]{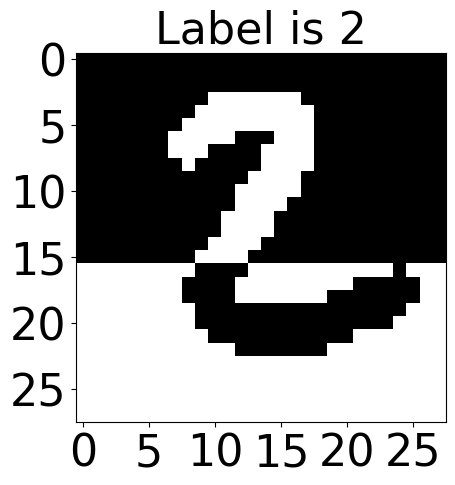}
\includegraphics[scale=0.06]{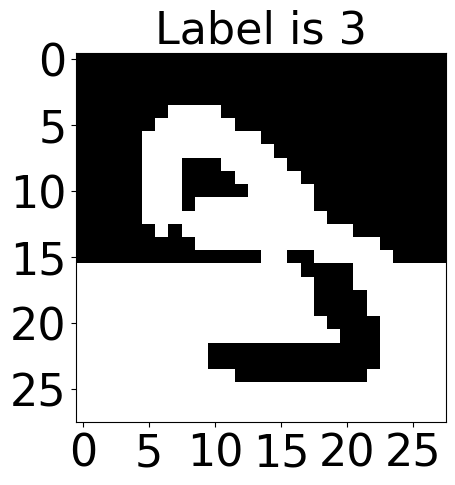}
\includegraphics[scale=0.06]{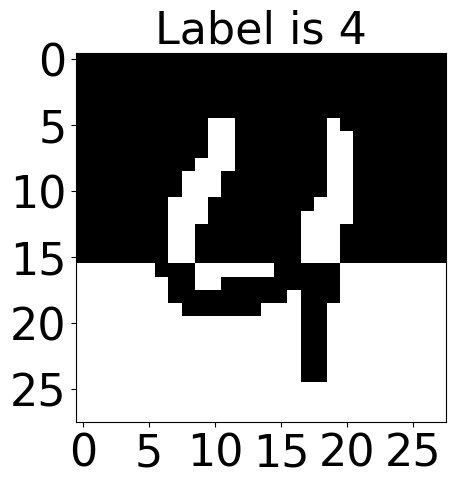}
\includegraphics[scale=0.06]{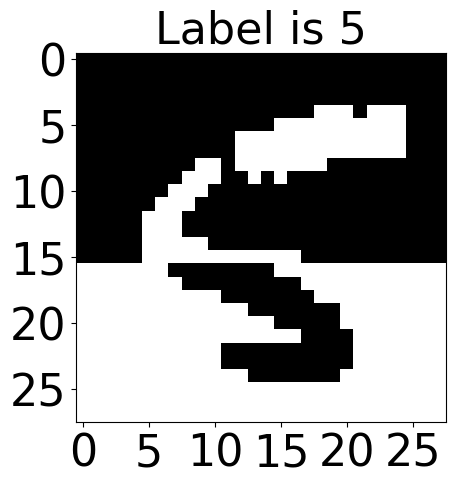}
\includegraphics[scale=0.06]{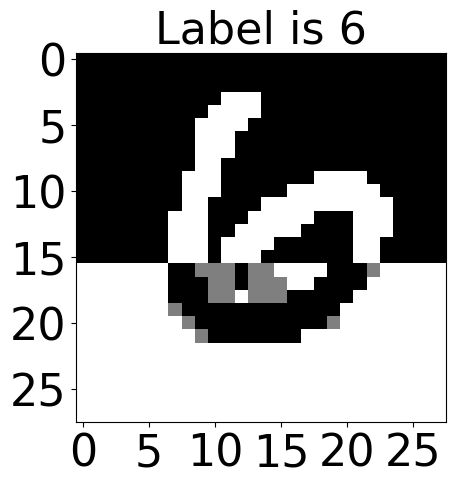}
\includegraphics[scale=0.06]{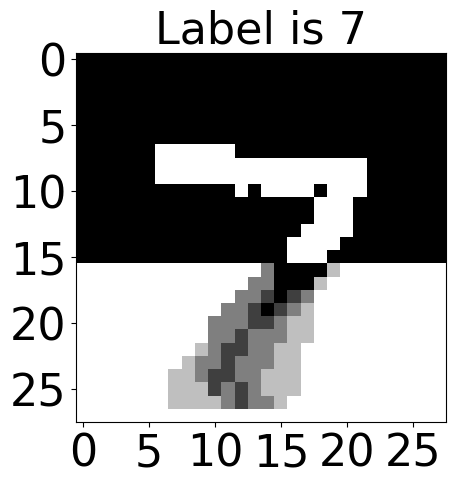}
\includegraphics[scale=0.06]{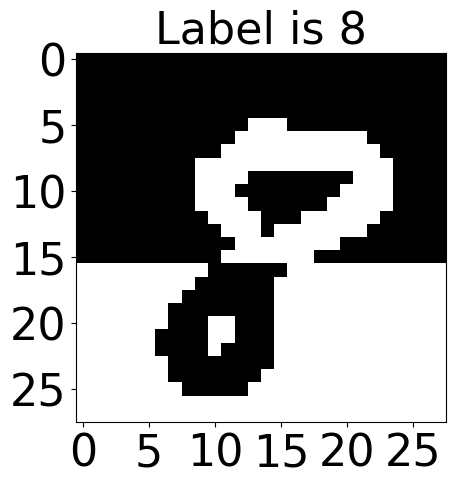}
\includegraphics[scale=0.06]{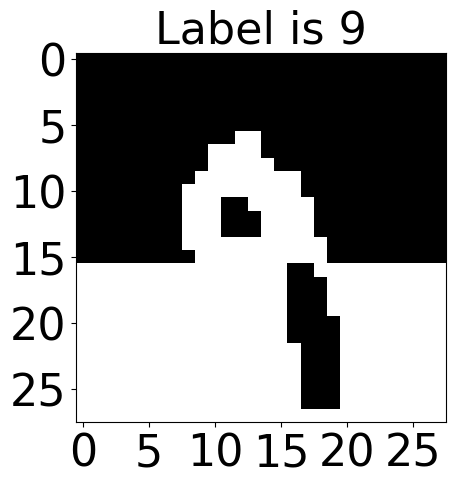}
\\
\includegraphics[scale=0.06]{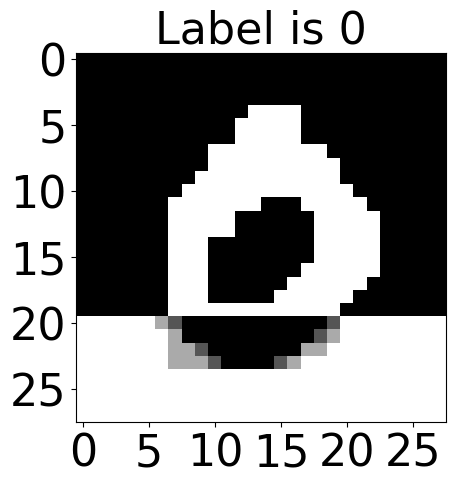}
\includegraphics[scale=0.06]{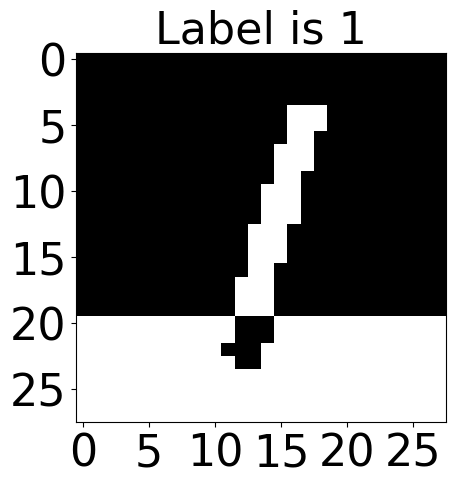}
\includegraphics[scale=0.06]{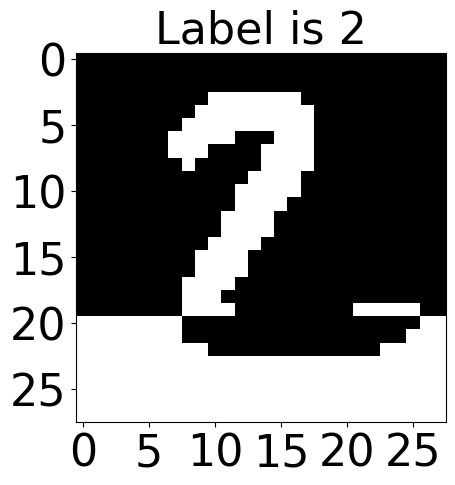}
\includegraphics[scale=0.06]{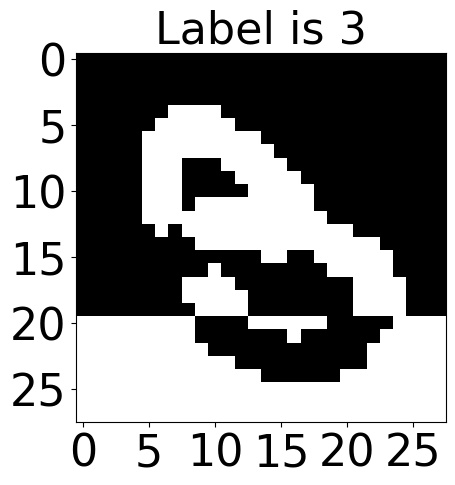}
\includegraphics[scale=0.06]{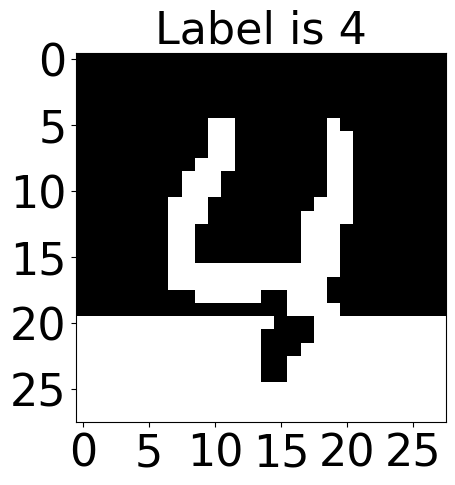}
\includegraphics[scale=0.06]{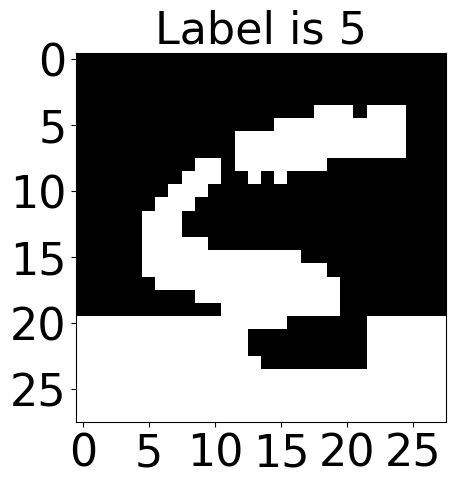}
\includegraphics[scale=0.06]{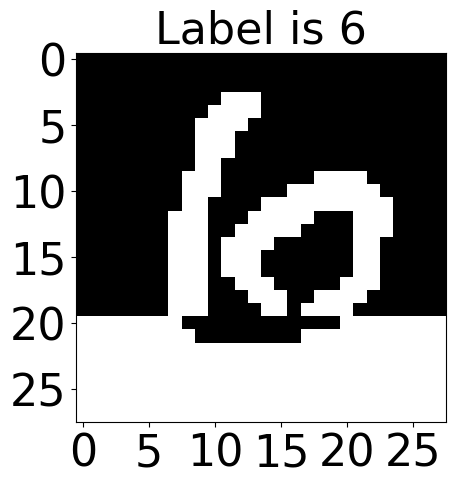}
\includegraphics[scale=0.06]{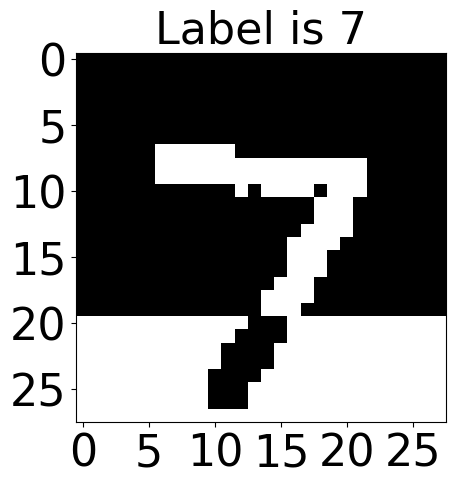}
\includegraphics[scale=0.06]{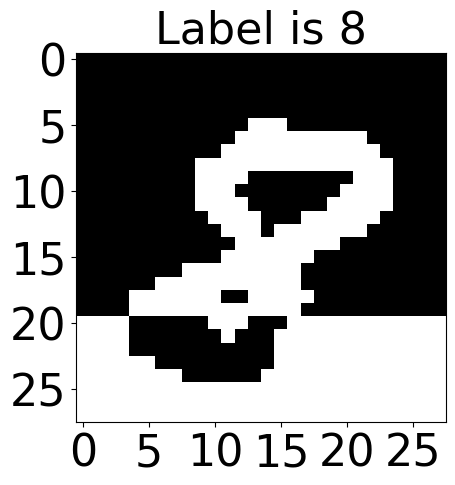}
\includegraphics[scale=0.06]{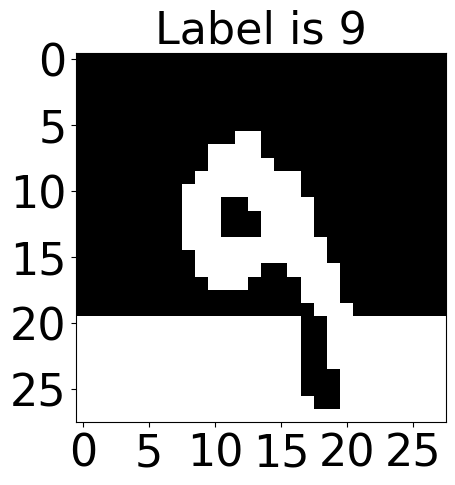}
\\
\includegraphics[scale=0.06]{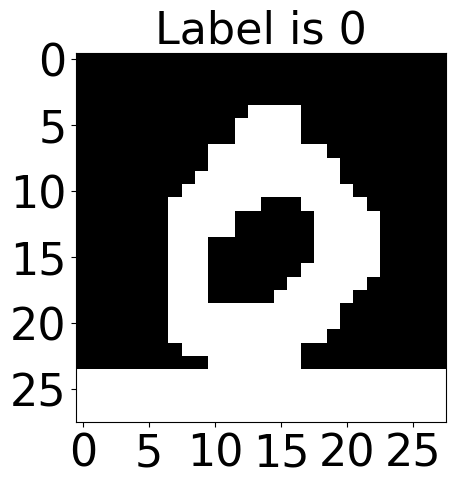}
\includegraphics[scale=0.06]{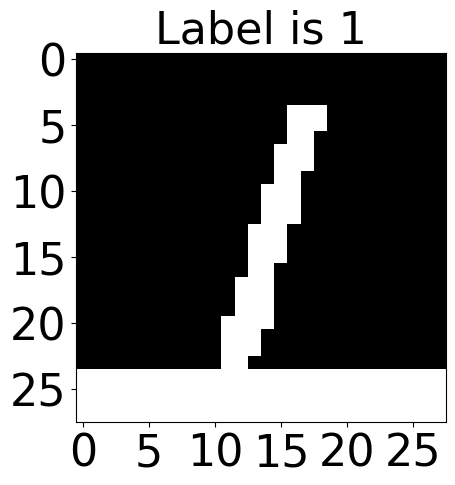}
\includegraphics[scale=0.06]{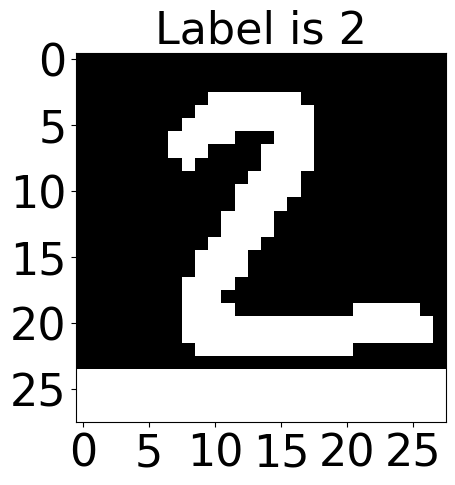}
\includegraphics[scale=0.06]{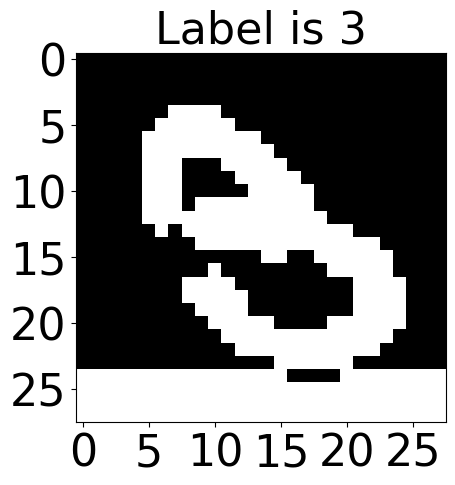}
\includegraphics[scale=0.06]{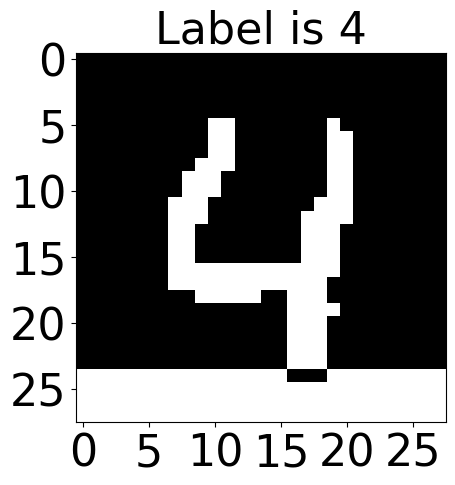}
\includegraphics[scale=0.06]{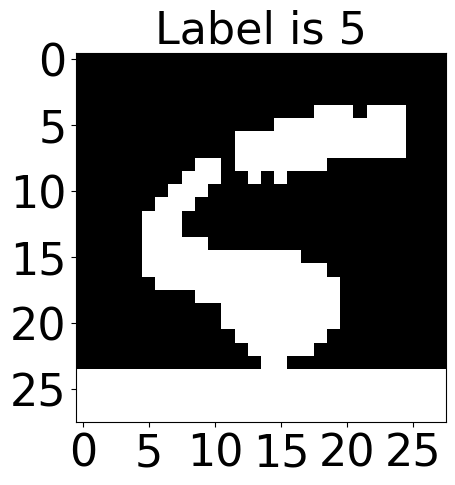}
\includegraphics[scale=0.06]{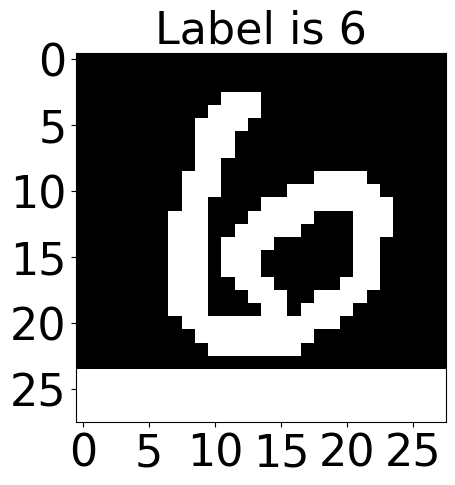}
\includegraphics[scale=0.06]{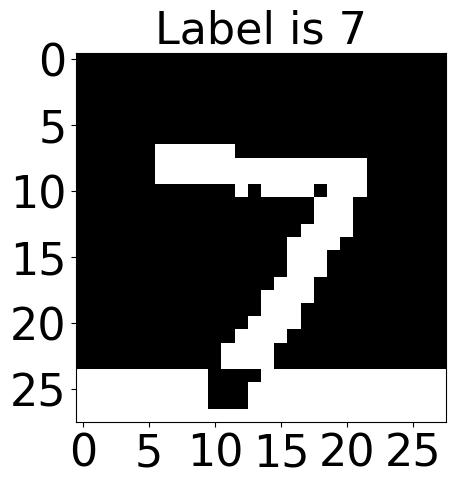}
\includegraphics[scale=0.06]{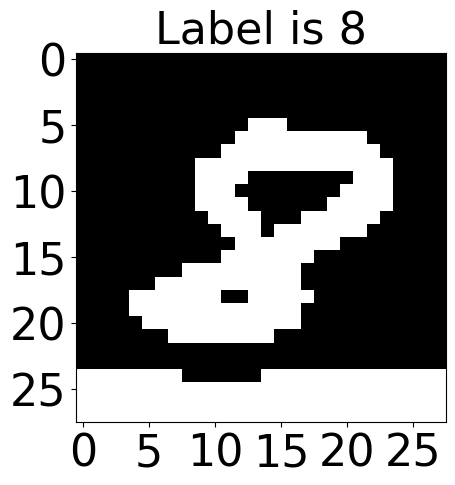}
\includegraphics[scale=0.06]{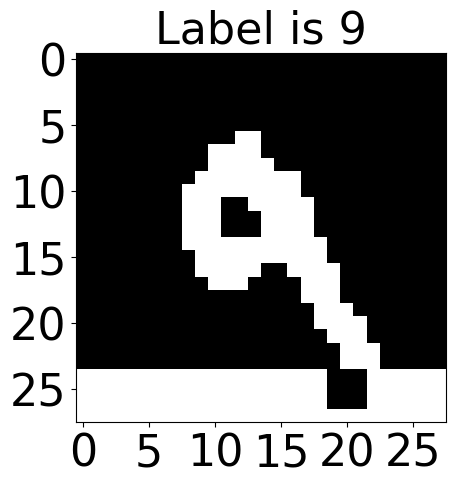}
\end{center}
\caption{The upper black-background area of each image in the $i$th row visualises the first $4i\times 28$ pixel values, approximately worth $14.3i\%$, of a test image. Given the partial information on the test image, the lower masked white-background area is generated using all the 60k training images. We inverted the colours of the generated pixel-values for visibility. Each test image is the first image of the digit in the test dataset. We again normalised the generated pixels.}
\label{fig:generation_from_pixels}
\end{figure}
Figure \ref{fig:generation} shows the images generated from each digit using the following equation, for all $j (1\leq j\leq 28\times 28)$.
{\small
\begin{align*}
p(Pixel_{j}|Digit_{i})=\frac{\sum_{k=1}^{70k}p(Digit_{i}|d_{k})p(Pixel_{j}|d_{k})}{\sum_{k=1}^{70k}p(Digit_{i}|d_{k})}
\end{align*}
}
Each image can be seen as the average of all the images of the same digit. The pixel value of each pixel is the average of the all 70k images. Figure \ref{fig:generation_from_pixels} shows the entire test images generated from their partial pixel values using the following equation, for all $j (I< j\leq 28\times 28)$.
{\small
\begin{align*}
p(Pix_{j}|Pix_{1},...,Pix_{I})=\frac{\sum_{k=1}^{60k}p(Pix_{j}|d_{k})\prod_{i=1}^{I}p(Pix_{i}|d_{k})}{\sum_{k=1}^{60k}\prod_{i=1}^{I}p(Pix_{i}|d_{k})}
\end{align*}
} 
Figure \ref{fig:generation_from_pixels} shows that, given the upper black-background area extracted from a test image, the lower white-background area was generated using all the 60k training images. The top row shows that the images of two and six appear to successfully generate the correct digits even with the 112 pixels having very few clues about the digits. All the other images in the first row appear to be an average training image, since no clue about the digits is included in the 112 pixel values. The fourth row shows that reasonable images are generated from the 448 pixel values, approximately worth 57\%, for all the ten test images.
\section{Discussion}
Inspired by Bayesian approaches to brain function in neuroscience, we asked how reasoning and learning can be given the same probabilistic account. We simply modelled how data cause symbolic knowledge in terms of its satisfiability in formal logic. The underlying idea is that reasoning is a process of deriving symbolic knowledge from data via abstraction, i.e., selective ignorance. We theoretically showed that the theory of inference generalises the classical and empirical consequence relations. We empirically showed that it not only generalises a sort of k-nearest neighbour method but also outperforms a k-nearest neighbour method in AUC on the MNIST dataset.
\bibliography{btx_kido}
\end{document}